\documentclass{article} 
\usepackage{iclr2022_conference,times}

\usepackage{amsmath,amsfonts,bm}

\def\eqref#1{(\ref{#1})}

\def\1{\bm{1}}

\DeclareMathAlphabet{\mathsfit}{\encodingdefault}{\sfdefault}{m}{sl}
\SetMathAlphabet{\mathsfit}{bold}{\encodingdefault}{\sfdefault}{bx}{n}

\DeclareMathOperator*{\argmin}{arg\,min}

\usepackage[colorlinks,linkcolor=red,filecolor=blue,citecolor=blue,urlcolor=blue]{hyperref}
\usepackage{url}
\usepackage{booktabs}       
\usepackage{amsfonts}       
\usepackage{amsmath}
\usepackage{graphicx}
\usepackage{amsthm}
\usepackage{amssymb}
\usepackage{multirow}
\usepackage{makecell}
\usepackage{algpseudocode}
\usepackage{algorithm}
\usepackage{amssymb}%
\usepackage{mathtools}
\usepackage{stmaryrd}
\usepackage[T1]{fontenc}
\usepackage{xargs}
\usepackage{xcolor}
\usepackage{wrapfig}

\newtheorem{Lemma}{Lemma}

\newtheorem{definition}{Definition}
\newtheorem{Theorem}{Theorem}

\newtheorem{Corollary}{Corollary}

\newtheorem{assumption}{Assumption}
\newtheorem{Remark}{Remark}
\newtheorem*{Lemma*}{Lemma}
\newtheorem*{Theorem*}{Theorem}
\newtheorem*{Corollary*}{Corollary}

\newcommand{\beq}{\begin{equation}}
\newcommand{\eeq}{\end{equation}}
\newcommand{\eqdef}{\mathrel{\mathop:}=}
\def\EE{\mathbb{E}}
\newcommand{\norm}[1]{\left\Vert #1 \right\Vert}

\def\maxiter{T}

\newcommand{\algo}{\textsc{Comp-AMS}}

\title{\centering On Distributed Adaptive Optimization with Gradient Compression}

    \author{\hspace{4.1cm} Xiaoyun Li, Belhal Karimi, Ping Li\vspace{0.1in} \\
\hspace{4.8cm}Cognitive Computing Lab\\
\hspace{5.5cm}Baidu Research\\
\hspace{3.8cm}10900 NE 8th St. Bellevue, WA 98004, USA\\
\hspace{2cm} \texttt{\texttt{\{xiaoyunli,belhalkarimi,liping11\}@baidu.com}}
}

\iclrfinalcopy 
\begin{document}

\maketitle

\begin{abstract}\vspace{-0.1in}

\noindent\footnote{Published at ICLR 2022. Submission available to public in \url{www.openreview.net} since Sept. 2021.}We study \algo, a distributed optimization framework based on gradient averaging and adaptive AMSGrad algorithm. Gradient compression with error feedback is applied to reduce the communication cost in the gradient transmission process. Our convergence analysis of \algo\ shows that such compressed gradient averaging strategy yields same convergence rate as standard AMSGrad, and also exhibits the linear speedup effect w.r.t. the number of local workers. Compared with recently proposed protocols on distributed adaptive methods, \algo\ is simple and convenient. Numerical experiments are conducted to justify the theoretical findings, and demonstrate that the proposed method can achieve same test accuracy as the full-gradient AMSGrad with substantial communication savings. With its simplicity and efficiency, \algo\ can serve as a useful distributed training framework for adaptive gradient methods.\vspace{-0.1in}
\end{abstract}

\section{Introduction}\label{sec:introduction}

Deep neural network has achieved the state-of-the-art learning performance on numerous AI applications, e.g., computer vision and natural language processing~\citep{Proc:Graves_ICASSP13,Proc:GAN_NIPS14,Proc:Resnet_CVPR16,NLP_review18,sentiment_review18}, reinforcement learning~\citep{Arxiv:MnihKSGAWR13,Article:Levine_JMLR16,AlphaGo_17}, recommendation systems~\citep{Proc:Covington_2016}, computational advertising~\citep{Proc:Zhao_CIKM19,Proc:Xu_SIGMOD21,Arxiv:Zhao2022}, etc. With the increasing size of data and growing complexity of deep neural networks, standard single-machine training procedures encounter at least two major challenges:
\begin{itemize}
    \item Due to the limited computing power of a single-machine, processing the massive number of data samples takes a long time---training is too slow. Many real-world applications cannot afford spending days or even weeks on training.

    \item In many scenarios, data are stored on multiple servers, possibly at different locations, due to the storage constraints (massive user behavior data, Internet images, etc.) or privacy reasons~\citep{Proc:Chang18}.
    Hence, transmitting data among servers might be costly.
\end{itemize}

\textit{Distributed learning} framework has been commonly used to tackle the above two issues. Consider the distributed optimization task where $n$ workers jointly solve the following optimization problem
\begin{equation}\label{eq:opt}
\min_{\theta} f(\theta) \eqdef \min_{\theta} \frac{1}{n} \sum_{i=1}^n f_i(\theta)= \frac{1}{n} \sum_{i=1}^n \mathbb E_{x\sim \mathcal X_i}[F_i(\theta;x)],
\end{equation}
where the non-convex function $f_i$ represents the average loss over the local data samples for worker $i \in [n]$, and $\theta\in\mathbb R^d$ the global model parameter. $\mathcal X_i$ is the data distribution on each local node. In the classical centralized distributed setting, in each iteration the central server uniformly randomly assigns the data to $n$ local workers ($\mathcal X_i$'s are the same), at which the gradients of the model are computed in parallel. Then the central server aggregates the local gradients, updates the global model (e.g., by stochastic gradient descent (SGD)), and transmits back the updated model to the local nodes for subsequent gradient computation. The scenario where $\mathcal X_i$'s are different gives rise to the recently proposed Federated Learning (FL)~\citep{mcmahan2017communication} framework, which will not be the major focus of this work. As we can see, distributed training naturally solves aforementioned issues: 1) We use $n$ computing nodes to train the model, so the time per training epoch can be largely reduced; 2) There is no need to transmit the local data to central server. Besides, distributed training also provides stronger error tolerance since the training process could continue even one local machine breaks down. As a result of these advantages, there has been a surge of study and applications on distributed systems~\citep{nedic2009distributed,boyd2011distributed,duchi2011dual,Arxiv:Goyal17,hong2017prox,koloskova2019decentralized,lu2019gnsd}.

\textbf{Gradient compression.} Among many optimization strategies, SGD is still the most popular prototype in distributed training for its simplicity and effectiveness~\citep{chilimbi2014project,Proc:Agrawal_NIPS19,mikami2018massively}. Yet, when the deep learning model is very large, the communication between local nodes and central server could be expensive, and the burdensome gradient transmission would slow down the whole training system. Thus, reducing the communication cost in distributed SGD has become an active topic, and an important ingredient of large-scale distributed systems (e.g.,~\cite{Proc:Seide14}). Solutions based on quantization, sparsification and other compression techniques of the local gradients have been proposed, e.g.,~\cite{aji2017sparse,alistarh2017qsgd,de2017understanding,wen2017terngrad,bernstein2018signsgd,stich2018sparsified,wangni2018gradient,Proc:Ivkin_NIPS19,yang2019swalp,Arxiv:fedsketch20}. However, it has been observed both theoretically and empirically~\citep{stich2018sparsified,ajalloeian2020analysis}, that directly updating with the compressed gradients usually brings non-negligible performance downgrade in terms of convergence speed and accuracy. To tackle this problem, studies (e.g.,~\cite{stich2018sparsified,karimireddy2019error}) show that the technique of \textit{error feedback} can to a large extent remedy the issue of such gradient compression, achieving the same convergence rate as full-gradient SGD.

\textbf{Adaptive optimization.} In recent years, adaptive optimization algorithms (e.g., AdaGrad~\citep{Duchi10-adagrad}, Adam~\citep{kingma2014adam} and AMSGrad~\citep{reddi2019convergence}) have become popular because of their superior empirical performance. These methods use different implicit learning rates for different coordinates that keep changing adaptively throughout the training process, based on the learning trajectory. In many cases, adaptive methods have been shown to converge faster than SGD, sometimes with better generalization as well. Nevertheless, the body of literature that extends adaptive methods to distributed training is still fairly limited. In particular, even the simple gradient averaging approach, though appearing standard, has not been analyzed for adaptive optimization algorithms. Given that distributed SGD with compressed gradient averaging can match the performance of standard SGD, one natural question is: is it also true for adaptive methods? In this work, we fill this gap formally, by analyzing \algo, a distributed adaptive optimization framework using the gradient averaging protocol, with communication-efficient gradient compression. Our method has been implemented in the PaddlePaddle  platform (\url{www.paddlepaddle.org.cn}).

\textbf{Our contributions.} We study a simple algorithm design leveraging the \emph{adaptivity} of AMSGrad and the computational virtue of \emph{local gradient compression}:
\begin{itemize}
\item We propose \algo, a synchronous distributed adaptive optimization framework based on global averaging with gradient compression, which is efficient in both communication and memory as no local moment estimation is needed. We consider the BlockSign and Top-$k$ compressors, coupled with the error-feedback technique to compensate for the bias implied by the compression step for fast convergence.

\item We provide the convergence analysis of distributed \algo\ (with $n$ workers) in smooth non-convex optimization. In the special case of $n=1$ (single machine), similar to SGD, gradient compression with error feedback in adaptive method achieves the same convergence rate $\mathcal O(\frac{1}{\sqrt T})$ as the standard full-gradient counterpart. Also, we show that with a properly chosen learning rate, \algo\ achieves $\mathcal O(\frac{1}{\sqrt{nT}})$ convergence, implying a linear speedup in terms of the number of local workers to attain a stationary point.

\item Experiments are conducted on various training tasks on image classification and sentiment analysis to validate our theoretical findings on the linear speedup effect. Our results show that \algo\ has comparable performance with other distributed adaptive methods, and approaches the accuracy of full-precision AMSGrad with a substantially reduced communication cost. Thus, it can serve as a convenient distributed training strategy in practice.

\end{itemize}

\section{Related Work}\label{sec:related}

\subsection{Distributed SGD with Compressed Gradients}

\textbf{Quantization.} To reduce the expensive communication in large-scale distributed SGD training systems, extensive works have considered various compression techniques applied to the gradient transaction procedure. The first strategy is quantization.~\cite{Proc:8-bit_ICLR16} condenses 32-bit floating numbers into 8-bits when representing the gradients.~\cite{Proc:Seide14,bernstein2018signsgd,Proc:Bernstein_ICLR19,karimireddy2019error} use the extreme 1-bit information (sign) of the gradients, combined with tricks like momentum, majority vote and memory. Other quantization-based methods include QSGD~\citep{alistarh2017qsgd,Proc:Zhang_ICML17,Proc:Wu_ICML18} and LPC-SVRG~\citep{Proc:Yu_AISTATS19}, leveraging unbiased stochastic quantization. Quantization has been successfully applied to industrial-level applications, e.g.,~\cite{Proc:Xu_SIGMOD21}. The saving in communication of quantization methods is moderate: for example, 8-bit quantization reduces the cost to 25\% (compared with 32-bit full-precision). Even in the extreme 1-bit case, the largest compression ratio is around $1/32\approx 3.1\%$.

\textbf{Sparsification.} Gradient sparsification is another popular solution which may provide higher compression rate. Instead of commuting the full gradient, each local worker only passes a few coordinates to the central server and zeros out the others. Thus, we can more freely choose higher compression ratio (e.g., 1\%, 0.1\%), still achieving impressive performance in many applications~\citep{Proc:Lin_ICLR18}. Stochastic sparsification methods, including uniform and magnitude based sampling~\citep{wangni2018gradient}, select coordinates based on some sampling probability, yielding unbiased gradient compressors with proper scaling. Deterministic methods are simpler, e.g., Random-$k$, Top-$k$~\citep{stich2018sparsified,shi2019convergence} (selecting $k$ elements with largest magnitude), Deep Gradient Compression~\citep{Proc:Lin_ICLR18}, but usually lead to biased gradient estimation. More applications and analysis of compressed distributed SGD can be found in~\citet{alistarh2018convergence,jiang2018linear,Proc:Jiang_SIGMOD18,Proc:Shen_ICML18,Proc:Basu_NIPS19}, among others.

\textbf{Error Feedback (EF).} Biased gradient estimation, which is a consequence of many aforementioned methods (e.g., signSGD, Top-$k$), undermines the model training, both theoretically and empirically, with slower convergence and worse generalization~\citep{ajalloeian2020analysis,Arxiv:Beznosikov20}. The technique of \textit{error feedback} is able to ``correct for the bias'' and fix the convergence issues.
In this procedure, the difference between the true stochastic gradient and the compressed one is accumulated locally, which is then added back to the local gradients in later iterations.~\cite{stich2018sparsified,karimireddy2019error} prove the $\mathcal O(\frac{1}{T})$ and $\mathcal O(\frac{1}{\sqrt T})$ convergence rate of EF-SGD in strongly convex and non-convex setting respectively, matching the rates of vanilla SGD~\citep{nemirovski2009robust,ghadimi2013stochastic}. More recent works on the convergence rate of SGD with error feedback include~\cite{Article:Stich_arxiv19,Proc:Zheng_NIPS19,Proc:Richtarik_NeurIPS21}, etc.

\subsection{Adaptive Optimization}

\begin{wrapfigure}{r}{0.5\linewidth}
\vspace{-0.2in}
\begin{minipage}{\linewidth}
\begin{algorithm}[H]
\caption{\textsc{AMSGrad}~\citep{reddi2019convergence}} \label{alg:amsgrad}
\begin{algorithmic}[1]
\State{\textbf{Input}: parameters $\beta_1$, $\beta_2$, $\epsilon$, learning rate $\eta_t$ }
\State{\textbf{Initialize:} $\theta_{1} \in \mathbb R^d$, $m_0=v_{0} = \bm{0} \in \mathbb R^{d}$}
\vspace{0.03in}
\State{\textbf{for $t=1, \ldots, T$ do}}
\State{\hspace{0.2in}Compute stochastic gradient $g_t$ at $\theta_t$}
\State{\hspace{0.2in}$m_t = \beta_1 m_{t-1} + (1 - \beta_1) g_t$}
\State{\hspace{0.2in}$v_t = \beta_2 v_{t-1} + (1 - \beta_2) g_t^2$ }
\State{\label{line:maxop}\hspace{0.2in}$\hat{v}_t = \max( \hat{v}_{t-1} , v_t )$ }
\State{\hspace{0.2in}$\theta_{t+1} = \theta_t - \eta_t \frac{m_t}{ \sqrt{\hat{v}_t +\epsilon} }$}
\State{\textbf{end for}}
\end{algorithmic}
\end{algorithm}
\end{minipage}
\end{wrapfigure}

In each SGD update, all the coordinates share the same learning rate, which is either constant or decreasing through the iterations. Adaptive optimization methods cast different learning rates on each dimension. For instance, AdaGrad, developed in~\cite{Duchi10-adagrad}, divides the gradient elementwise by $\sqrt{\sum_{t=1}^T g_{t}^2}\in \mathbb R^d$, where $g_{t}\in \mathbb R^d$ is the gradient vector at time $t$ and $d$ is the model dimensionality. Thus, it intrinsically assigns different learning rates to different coordinates throughout the training---elements with smaller previous gradient magnitudes tend to move a larger step via larger learning rate. Other adaptive methods include AdaDelta~\citep{Proc:adadelta} and Adam~\citep{kingma2014adam}, which introduce momentum and moving average of second moment estimation into AdaGrad hence leading to better performances.
AMSGrad~\citep{reddi2019convergence} (Algorithm~\ref{alg:amsgrad}, which is the prototype in our paper), fixes the potential convergence issue of Adam. \cite{wang2019optimistic} and \cite{Proc:Zhou_NeurIPS20} improve the convergence and generalization of AMSGrad through optimistic acceleration and differential privacy.

Adaptive optimization methods have been widely used in training deep learning models in language, computer vision and advertising applications, e.g.,~\cite{,Arxiv:Choi_2019,Proc:LAMB_ICLR20,Arxiv:Zhang_ICLR21,Arxiv:Zhao2022}. In distributed setting,~\cite{nazari2019dadam,chen2021convergence} study decentralized adaptive methods, but communication efficiency was not considered. Mostly relevant to our work,~\cite{Arxiv:QAdam} proposes a distributed training algorithm based on Adam, which requires every local node to store a local estimation of the moments of the gradient. Thus, one has to keep extra two more tensors of the model size on each local worker, which may be less feasible in terms of memory particularly with large models. More recently,~\cite{Proc:1bitAdam} proposes an Adam pre-conditioned momentum SGD method. \cite{chen2020toward,karimireddy2020mime,Proc:Reddi_ICLR21} proposed local/global adaptive FL methods, which can be further accelerated via layer-wise adaptivity~\citep{Arxiv:fedlamb2021}.

\section{Communication-Efficient Adaptive Optimization}\label{sec:main}

\subsection{Gradient Compressors}

In this paper, we mainly consider deterministic $q$-deviate compressors defined as below.

\begin{assumption}\label{ass:quant} The gradient compressor $\mathcal C:\mathbb R^d\mapsto \mathbb R^d$ is $q$-deviate: for $\forall x\in\mathbb R^d$, $\exists$ $0\leq q < 1$ such that $\norm{\mathcal C(x)-x} \leq q \norm{x}$.
\end{assumption}
Larger $q$ indicates heavier compression, while smaller $q$ implies better approximation of the true gradient. $q=0$ implies $\mathcal C(x)=x$, i.e., no compression. In the following, we give two popular and efficient $q$-deviate compressors that will be adopted in this paper.

\begin{definition}[Top-$k$]\label{def:topk}
For $x\in\mathbb R^d$, denote $\mathcal S$ as the size-$k$ set of $i\in[d]$ with largest $k$ magnitude $|x_i|$. The \textbf{Top-$k$} compressor is defined as $\mathcal C(x)_i=x_i$, if $i\in\mathcal S$; $\mathcal C(x)_i=0$ otherwise.
\end{definition}

\begin{definition}[Block-Sign]\label{def:sign}
For $x\in\mathbb R^d$, define $M$ blocks indexed by $\mathcal B_i$, $i=1,...,M$, with $d_i\eqdef |\mathcal B_i|$. The \textbf{Block-Sign} compressor is defined as $\mathcal C(x)=[sign(x_{\mathcal B_1})\frac{\|x_{\mathcal B_1}\|_1}{d_1},..., sign(x_{\mathcal B_M}) \frac{\|x_{\mathcal B_M}\|_1}{d_M}]$, where $x_{\mathcal B_i}$ is the sub-vector of $x$ at indices $\mathcal B_i$.
\end{definition}

\begin{Remark}
It is well-known~\citep{stich2018sparsified} that for \textbf{Top-$k$}, $q^2=1-\frac{k}{d}$. For \textbf{Block-Sign}, by Cauchy-Schwartz inequality we have $q^2=1-\min_{i\in [M]} \frac{1}{d_i}$ where $M$ and $d_i$ are defined in Definition~\ref{def:sign}~\citep{Proc:Zheng_NIPS19}.
\end{Remark}

The intuition behind \textbf{Top-$k$} is that, it has been observed empirically that when training many deep models, most gradients are typically very small, and gradients with large magnitude contain most information. The \textbf{Block-Sign} compressor is a simple extension of the 1-bit \textbf{SIGN} compressor~\citep{Proc:Seide14,bernstein2018signsgd}, adapted to different gradient magnitude in different blocks, which, for neural nets, are usually set as the distinct network layers. The scaling factor in Definition~\ref{def:sign} is to preserve the (possibly very different) gradient magnitude in each layer. In principle, \textbf{Top-$k$} would perform the best when the gradient is effectively sparse, while \textbf{Block-Sign} compressor is favorable by nature when most gradients have similar magnitude within each layer.

\subsection{\algo : Distributed Adaptive Training by Gradient Aggregation}

We present in Algorithm~\ref{alg:sparsams} the proposed communication-efficient distributed adaptive method in this paper, \algo. This framework can be regarded as an analogue to the standard synchronous distributed SGD: in each iteration, each local worker transmits to the central server the compressed stochastic gradient computed using local data. Then the central server takes the average of local gradients, and performs an AMSGrad update. In Algorithm~\ref{alg:sparsams}, lines 7-8 depict the error feedback operation at local nodes. $e_{t,i}$ is the accumulated error from gradient compression on the $i$-th worker up to time $t-1$. This residual is added back to $g_{t,i}$ to get the ``corrected'' gradient. In Section~\ref{sec:theory} and Section~\ref{sec:experiment}, we will show that error feedback, similar to the case of SGD, also brings good convergence behavior under gradient compression in distributed AMSGrad.

\newpage

\textbf{Comparison with related methods.} Next, we discuss the differences between \algo\ and two recently proposed methods also trying to solve compressed distributed adaptive optimization.

\begin{itemize}
    \item \textbf{Comparison with~\cite{Arxiv:QAdam}.}\hspace{0.1in}\cite{Arxiv:QAdam} develops a quantized variant of Adam~\citep{kingma2014adam}, called ``QAdam''. In this method, each worker keeps a local copy of the moment estimates, commonly noted $m$ and $v$, and compresses and transmits the ratio $\frac{m}{v}$ as a whole to the server. Their method is hence very much like the compressed distributed SGD, with the exception that the ratio $\frac{m}{v}$ plays the role of the gradient vector $g$ communication-wise. Thus, two local moment estimators are additionally required, which have same size as the deep learning model. In our \algo, the moment estimates $m$ and $v$ are kept and updated only at the central server, thus not introducing any extra variable (tensor) on local nodes during training (except for the error accumulator). Hence, \algo\ is not only effective in communication reduction, but also efficient in terms of memory (space), which is feasible when training large-scale learners like BERT and CTR prediction models, e.g.,~\cite{Proc:BERT,Proc:Xu_SIGMOD21}, to lower the hardware consumption in practice. Additionally, the convergence rate in \cite{Arxiv:QAdam} does not improve linearly with $n$, while we prove the linear speedup effect of \algo.

    \item \textbf{Comparison with~\cite{Proc:1bitAdam}}\hspace{0.1in}The recent work~\citep{Proc:1bitAdam} proposes ``1BitAdam''. They first run some warm-up training steps using standard Adam, and then store the second moment moving average $v$. Then, distributed Adam training starts with $v$ frozen. Thus, 1BitAdam is actually more like a distributed momentum SGD with some pre-conditioned coordinate-wise learning rates. The number of warm-up steps also needs to be carefully tuned, otherwise bad pre-conditioning may hurt the learning performance. Our \algo\ is simpler, as no pre-training is needed. Also, 1BitAdam requires extra tensors for $m$ locally, while \algo\  does not need additional local memory.
\end{itemize}

\begin{algorithm}[tb]
\caption{Distributed \algo\ with error feedback (EF)} \label{alg:sparsams}
\begin{algorithmic}[1]
\State{\textbf{Input}: parameters $\beta_1$, $\beta_2$, $\epsilon$, learning rate $\eta_t$ }
\State{\textbf{Initialize}: central server parameter $\theta_{1} \in \mathbb R^d \subseteq \mathbb R^d$; $e_{1,i}=\bm{0}$ the error accumulator for each worker; $m_0=\bm{0}$, $v_0=\bm{0}$, $\hat v_0=\bm{0}$}
\vspace{0.03in}
\State{\textbf{for $t=1, \ldots, T$ do}}
\State{\quad\textbf{parallel for worker $i \in [n]$ do}:}
\State{\quad\quad  Receive model parameter $\theta_{t}$ from central server}
\State{\quad\quad  Compute stochastic gradient $g_{t,i}$ at $\theta_t$}
\State{\quad\quad  Compute the compressed gradient $\tilde g_{t,i}=\mathcal C(g_{t,i}+e_{t,i})$ \label{line:topk} }
\State{\quad\quad  Update the error $e_{t+1,i}=e_{t,i}+g_{t,i}-\tilde g_{t,i}$}
\State{\quad\quad  Send $\tilde g_{t,i}$ back to central server}
\State{\quad \textbf{end parallel}}

\State{\quad\textbf{Central server do:}}
\State{\quad $\bar g_{t}=\frac{1}{n}\sum_{i=1}^n \tilde g_{t,i}$}
\State{\quad $m_t=\beta_1 m_{t-1}+(1-\beta_1)\bar g_t$}
\State{\quad $v_t=\beta_2 v_{t-1}+(1-\beta_2)\bar g_t^2$}
\State{\quad $\hat v_t=\max(v_t,\hat v_{t-1})$ \label{line:v}}
\State{\quad Update the global model $\theta_{t+1}=\theta_{t}-\eta_t\frac{m_t}{\sqrt{\hat v_t+\epsilon}}$}
\State{\textbf{end for}}
\end{algorithmic}
\end{algorithm}

\section{Convergence Analysis}\label{sec:theory}

For the convergence analysis of \algo\, we will make following additional assumptions.

\begin{assumption}\label{ass:smooth}(Smoothness)
For $\forall i \in [n]$, $f_i$ is  L-smooth: $\norm{\nabla f_i (\theta) - \nabla f_i (\vartheta)} \leq L \norm{\theta-\vartheta}$.
\end{assumption}

\begin{assumption}\label{ass:boundgrad}(Unbiased and bounded stochastic gradient)
For $\forall t >0$, $\forall i \in [n]$, the stochastic gradient is unbiased and uniformly bounded: $\EE[g_{t,i}] = \nabla f_i(\theta_t)$ and $\norm{g_{t,i}} \leq G$.
\end{assumption}

\begin{assumption}\label{ass:var}(Bounded variance)
For $\forall t >0$, $\forall i \in [n]$: (i) the \textbf{local variance} of the stochastic gradient is bounded: $\EE[\|g_{t,i} - \nabla f_i(\theta_t)\|^2] < \sigma^2$; (ii) the \textbf{global variance} is bounded by $\frac{1}{n}\sum_{i=1}^n\|\nabla f_i(\theta_t)-\nabla f(\theta_t)\|^2\leq \sigma_g^2$.
\end{assumption}

In Assumption~\ref{ass:boundgrad}, the uniform bound on the stochastic gradient is common in the convergence analysis of adaptive methods, e.g.,~\cite{reddi2019convergence,Arxiv:Zhou_18,Proc:Chen_ICLR19}. The global variance bound $\sigma_g^2$ in Assumption~\ref{ass:var} characterizes the difference among local objective functions, which, is mainly caused by different local data distribution $\mathcal X_i$ in \eqref{eq:opt}. In classical distributed setting where all the workers can access the same dataset and local data are assigned randomly, $\sigma_g^2\equiv 0$. While typical federated learning (FL) setting with $\sigma_g^2>0$ is not the focus of this present paper, we consider the global variance in our analysis to shed some light on the impact of non-i.i.d. data distribution in the federated setting for broader interest and future investigation.

We derive the following general convergence rate of \algo\ in the distributed setting. The proof is deferred to Appendix~\ref{app:proof}.

\begin{Theorem}  \label{theo:rate}
Denote $C_0=\sqrt{\frac{4(1+q^2)^3}{(1-q^2)^2}G^2+\epsilon}$, $C_1=\frac{\beta_1}{1-\beta_1}+\frac{2q}{1-q^2}$, $\theta^*=\argmin f(\theta)$ defined as (\ref{eq:opt}). Under Assumptions~\ref{ass:quant} to~\ref{ass:var}, with $\eta_t=\eta\leq \frac{\epsilon}{3C_0\sqrt{2L \max\{2L,C_1\}}}$, Algorithm~\ref{alg:sparsams} satisfies
\begin{align*}
    \frac{1}{T}\sum_{t=1}^T \mathbb E[\|\nabla f(\theta_t)\|^2]
    &\leq 2C_0\Big(\frac{\mathbb E[f(\theta_1)-f(\theta^*)]}{T\eta}+\frac{\eta L \sigma^2}{n\epsilon}+\frac{3\eta^2 LC_0C_1^2\sigma^2}{n\epsilon^2}  \\
    &\hspace{0.7in} + \frac{12\eta^2q^2LC_0\sigma_g^2}{(1-q^2)^2\epsilon^2}+\frac{ (1+C_1)G^2d}{T\sqrt\epsilon}+\frac{\eta (1+2C_1)C_1LG^2d}{T\epsilon} \Big).
\end{align*}
\end{Theorem}

The LHS of Theorem~\ref{theo:rate} is the expected squared norm of the gradient from a uniformly chosen iterate $t\in [T]$, which is a common convergence measure in non-convex optimization. From Theorem~\ref{theo:rate}, we see that the more compression we apply to the gradient vectors (i.e.,~larger $q$), the larger the gradient magnitude is, i.e.,~the slower the algorithm converges. This is intuitive as heavier compression loses more gradient information which would slower down the learner to find a good solution.

Note that, \algo\ with $n=1$ naturally reduces to the single-machine (sequential) AMSGrad (Algorithm~\ref{alg:amsgrad}) with compressed gradients instead of full-precision ones.~\cite{karimireddy2019error} specifically analyzed this case for SGD, showing that compressed single-machine SGD with error feedback has the same convergence rate as vanilla SGD using full gradients. In alignment with the conclusion in~\citet{karimireddy2019error}, for adaptive AMSGrad, we have a similar result.

\vspace{0.1in}

\begin{Corollary}\label{coro:mainsingle}
When $n=1$, under Assumption~\ref{ass:quant} to Assumption~\ref{ass:var}, setting the stepsize as $\eta = \min\{\frac{\epsilon}{3C_0\sqrt{2L \max\{2L,C_1\}}}, \frac{1}{\sqrt{\maxiter}}\}$, Algorithm~\ref{alg:sparsams} satisfies
\begin{align*}
    \frac{1}{T}\sum_{t=1}^T \mathbb E[\|\nabla f(\theta_t)\|^2]\leq \mathcal O(\frac{1}{\sqrt{T}}+ \frac{\sigma^2}{\sqrt T}+\frac{d}{T}).
\end{align*}
\end{Corollary}

Corollary~\ref{coro:mainsingle} states that with error feedback, single machine AMSGrad with biased compressed gradients can also match the convergence rate $\mathcal O(\frac{1}{\sqrt{T}}+\frac{d}{T})$ of standard AMSGrad~\citep{Arxiv:Zhou_18} in non-convex optimization. It also achieves the same rate $\mathcal O(\frac{1}{\sqrt T})$ of vanilla SGD~\citep{karimireddy2019error} when $T$ is sufficiently large. In other words, error feedback also fixes the convergence issue of using compressed gradients in AMSGrad.

\vspace{0.05in}
\textbf{Linear Speedup.} In Theorem~\ref{theo:rate}, the convergence rate is derived by assuming a constant learning rate. By carefully choosing a decreasing learning rate dependent on the number of workers, we have the following simplified statement.

\begin{Corollary}\label{coro:linear speedup}
Under the same setting as Theorem~\ref{theo:rate}, set $\eta = \min\{\frac{\epsilon}{3C_0\sqrt{2L \max\{2L,C_1\}}}, \frac{\sqrt n}{\sqrt{\maxiter}}\}$. The \algo\ iterates admit
\begin{align}
    \frac{1}{T}\sum_{t=1}^T \mathbb E[\|\nabla f(\theta_t)\|^2]\leq \mathcal O(\frac{1}{\sqrt{nT}}+\frac{\sigma^2}{\sqrt{nT}}+\frac{n(\sigma^2+\sigma_g^2)}{T}).  \label{label:eq:linear speedup}
\end{align}
\end{Corollary}

In Corollary~\ref{coro:linear speedup}, we see that the global variance $\sigma_g^2$ appears in the $\mathcal O(\frac{1}{T})$ term, which says that it asymptotically has no impact on the convergence. This matches the result of momentum SGD~\citep{Proc:Yu_ICML19}. When $T\geq \mathcal O(n^3)$ is sufficiently large, the third term in \eqref{label:eq:linear speedup} vanishes, and the convergence rate becomes $\mathcal O(\frac{1}{\sqrt{nT}})$. Therefore, to reach an $\mathcal O(\delta)$ stationary point, one worker ($n=1$) needs $T=\mathcal O(\frac{1}{\delta^2})$ iterations, while distributed training with $n$ workers requires only $T=\mathcal O(\frac{1}{n\delta^2})$ iterations, which is $n$ times faster than single machine training. That is, \algo\ has a linear speedup in terms of the number of the local workers. Such acceleration effect has also been reported for compressed SGD~\citep{jiang2018linear,Proc:Zheng_NIPS19} and momentum SGD~\citep{Proc:Yu_ICML19} with error feedback.

\section{Experiments}\label{sec:experiment}

In this section, we provide numerical results on several common datasets. Our main objective is to validate the theoretical results, and demonstrate that the proposed \algo\ can approach the learning performance of full-precision AMSGrad with significantly reduced communication costs.

\subsection{Datasets, Models and Methods}

Our experiments are conducted on various image and text datasets. The MNIST~\citep{mnist} contains 60000 training samples of $28\times 18$ gray-scale hand-written digits from 10 classes, and 10000 test samples. We train MNIST with a Convolutional Neural Network (CNN), which has two convolutional layers followed by two fully connected layers with ReLu activation. Dropout is applied after the max-pooled convolutional layer with rate 0.5. The CIFAR-10 dataset~\citep{cifar} consists of 50000 $32\times 32$ RGB natural images from 10 classes for training and 10000 images for testing, which is trained by LeNet-5~\citep{mnist}. Moreover, we also implement ResNet-18~\citep{Proc:Resnet_CVPR16} on this dataset. The IMDB movie review~\citep{imdb} is a popular binary classification dataset for sentiment analysis. Each movie review is tokenized by top-2000 most frequently appeared words and transformed into integer vectors, which is of maximal length 500. We train a Long-Short Term Memory (LSTM) network with a 32-dimensional embedding layer and 64 LSTM cells, followed by two fully connected layers before output. Cross-entropy loss is used for all the tasks. Following the classical distributed training setting, in each training iteration, data samples are uniformly randomly assigned to the workers.

We compare \algo\ with full-precision distributed AMSGrad, QAdam~\citep{Arxiv:QAdam} and 1BitAdam~\citep{Proc:1bitAdam}. For \algo, \textbf{Top-$k$} picks top 1\% gradient coordinates (i.e.,~sparsity 0.01). QAdam and 1BitAdam both use 1-bit quantization to achieve high compression. For MNIST and CIFAR-10, the local batch size on each worker is set to be 32. For IMDB, the local batch size is 16. The hyper-parameters in \algo\ are set as default $\beta_1=0.9$, $\beta_2=0.999$ and $\epsilon=10^{-8}$, which are also used for QAdam and 1BitAdam. For 1BitAdam, the epochs for warm-up training is set to be $1/20$ of the total epochs. For all methods, we tune the initial learning rate over a fine grid (see Appendix~\ref{app sec:resnet}) and report the best results averaged over three independent runs. Our experiments are performed on a GPU cluster with NVIDIA Tesla V100 cards.

\subsection{General Evaluation and Communication Efficiency}

The training loss and test accuracy on MNIST + CNN, CIFAR-10 + LeNet and IMDB + LSTM are reported in Figure~\ref{fig:loss acc}. We provide more results on larger ResNet-18 model in Appendix~\ref{app sec:resnet}. On CIFAR-10, we deploy a popular decreasing learning rate schedule, where the step size $\eta$ is divided by 10 at the 40-th and 80-th epoch, respectively. We observe:
\begin{itemize}
    \item On MNIST, all the methods can approach the training loss and test accuracy of full-precision AMSGrad. The 1BitAdam seems slightly better, but the gap is very small. On CIFAR-10, \algo\ with \textbf{Block-Sign} performs the best and matches AMSGrad in terms of test accuracy.

    \item On IMDB, \algo\ with \textbf{Top-$k$} has both the fastest convergence and best generalization compared with other compressed methods. This is because the IMDB text data is more sparse (with many padded zeros), where \textbf{Top-$k$} is expected to work better than sign. The 1BitAdam converges slowly. We believe one possible reason is that 1BitAdam is quite sensitive to the quality of the warm-up training. For sparse text data, the estimation of second moment $v$ is more unstable, making the strategy of freezing $v$ by warm-up less effective.
\end{itemize}

\newpage

\textbf{Communication Efficiency.} In Figure~\ref{fig:communication}, we plot the training loss and test accuracy against the number of bits transmitted to the central server during the distributed training process, where we assume that the full-precision gradient is represented using 32 bits per floating number. As we can see, \algo-\textbf{Top-$0.01$} achieves around 100x communication reduction, to attain similar accuracy as the full-precision distributed AMSGrad. The saving of \textbf{Block-Sign} is around 30x, but it gives slightly higher accuracy than \textbf{Top-$0.01$} on MNIST and CIFAR-10. In all cases, \algo\ can substantially reduce the communication cost compared with full-precision distributed AMSGrad, without losing accuracy.

\begin{figure}[t]
    \begin{center}
        \mbox{\hspace{-0.2in}
        \includegraphics[width=2in]{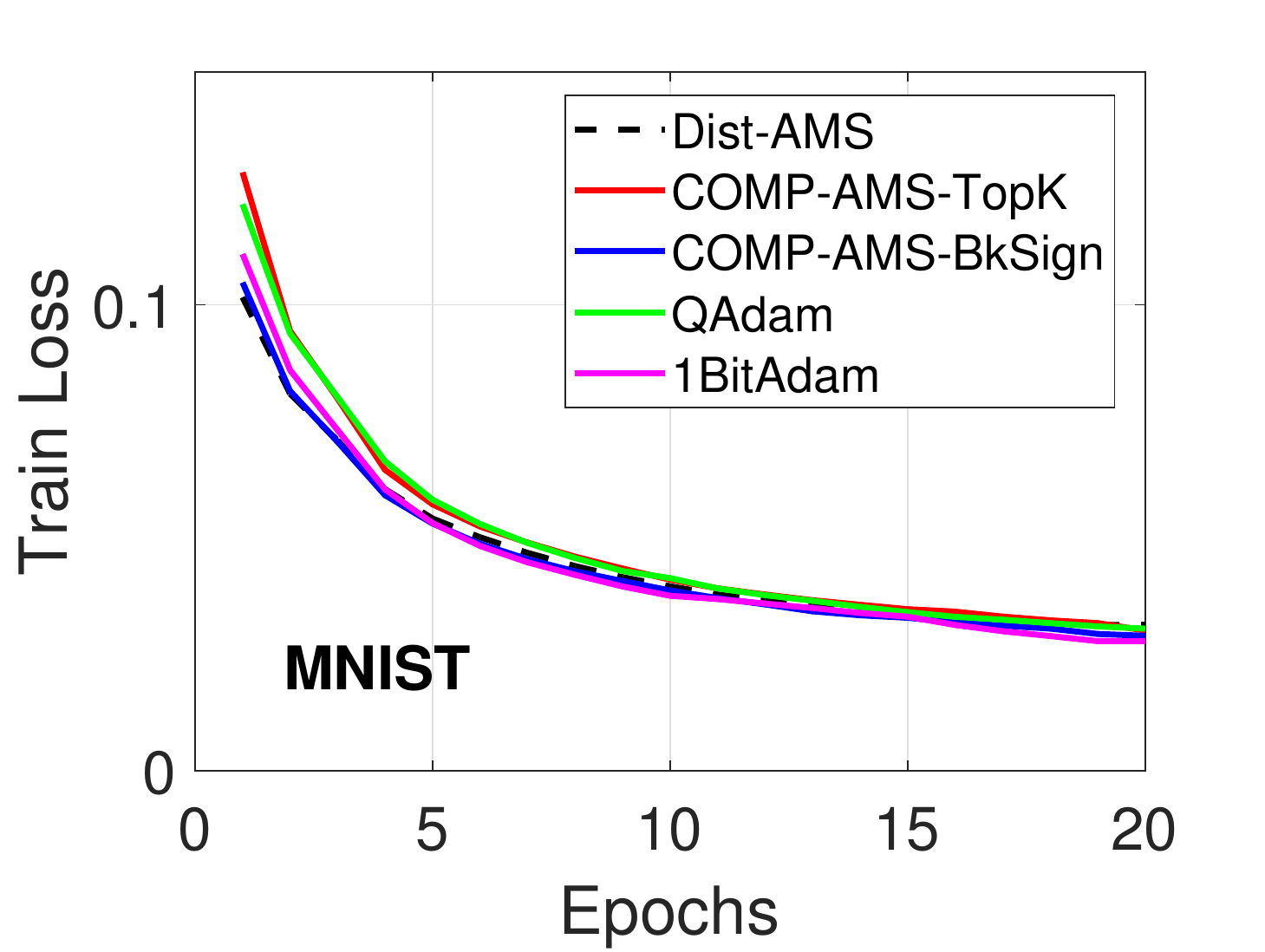}\hspace{-0.1in}
        \includegraphics[width=2in]{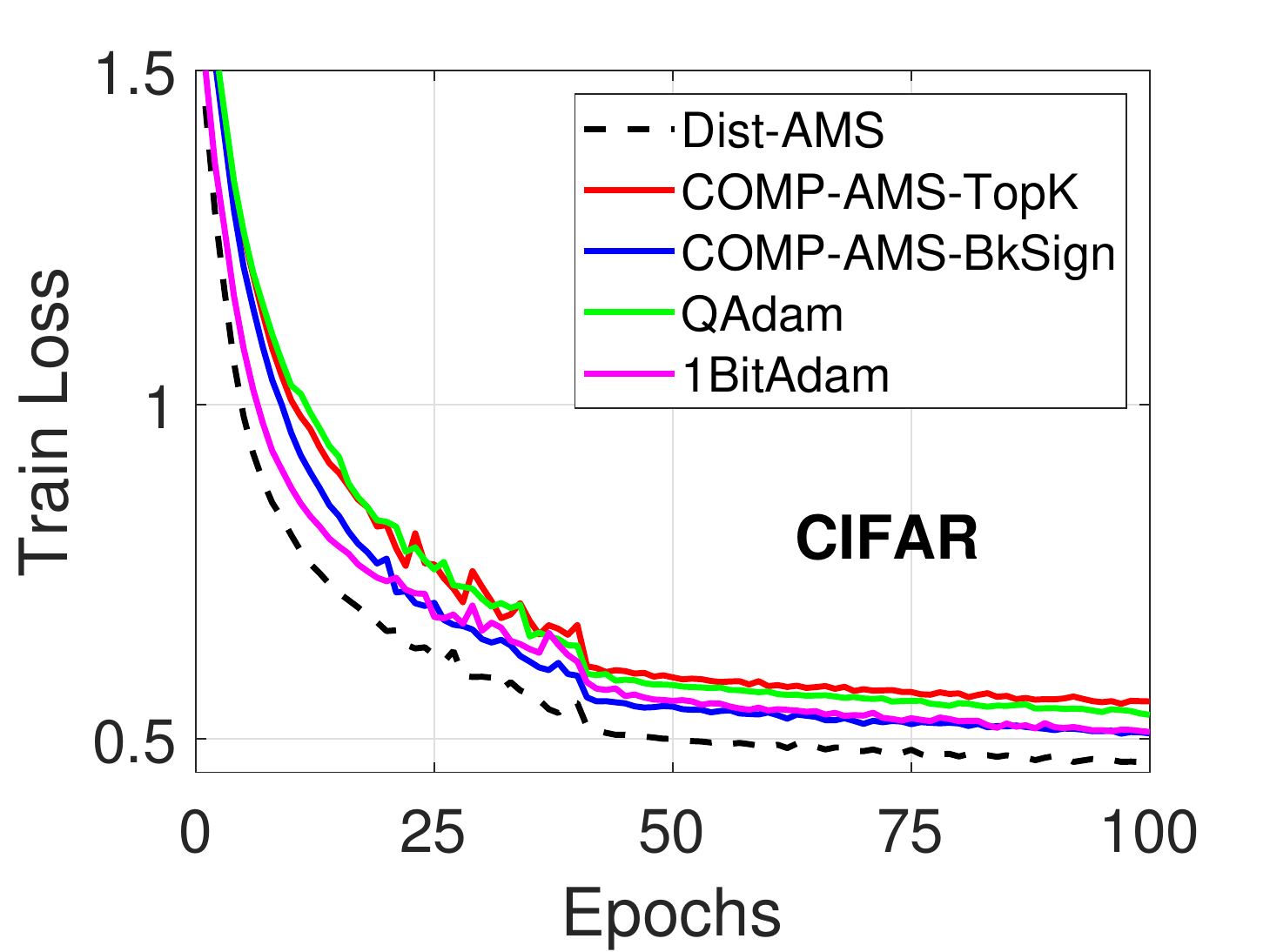}\hspace{-0.1in}
        \includegraphics[width=2in]{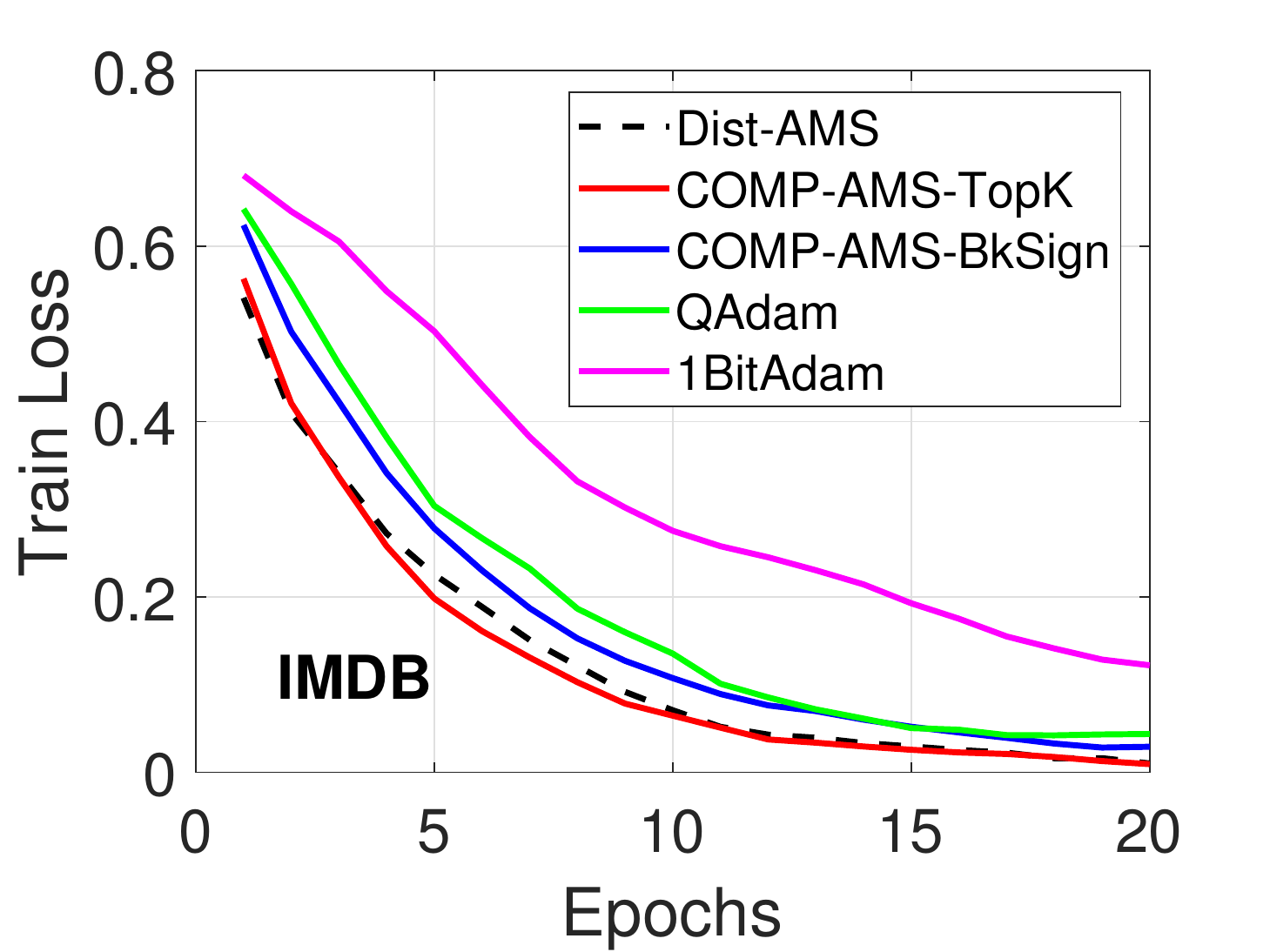}
        }
        \mbox{\hspace{-0.2in}
        \includegraphics[width=2in]{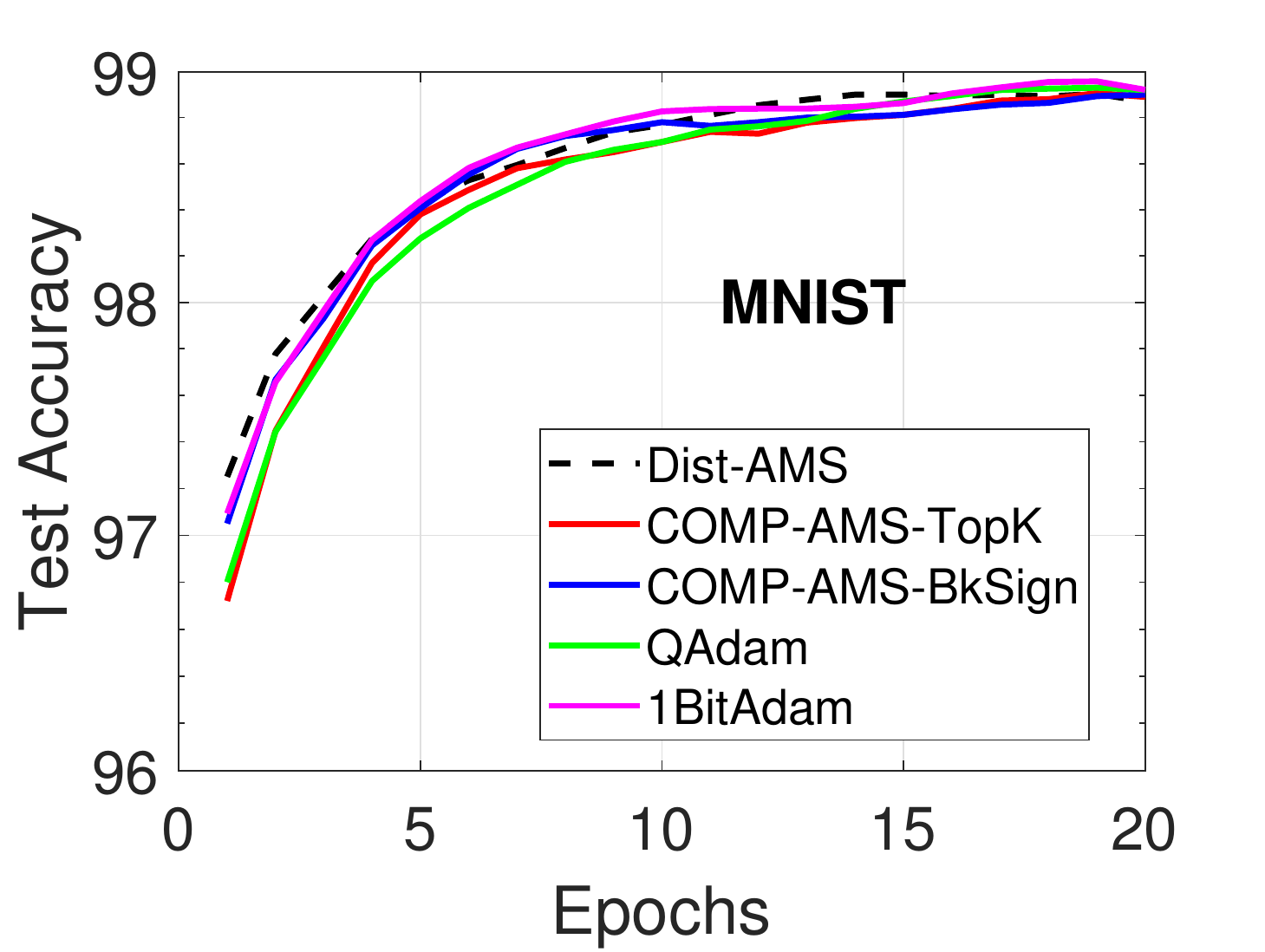}\hspace{-0.1in}
        \includegraphics[width=2in]{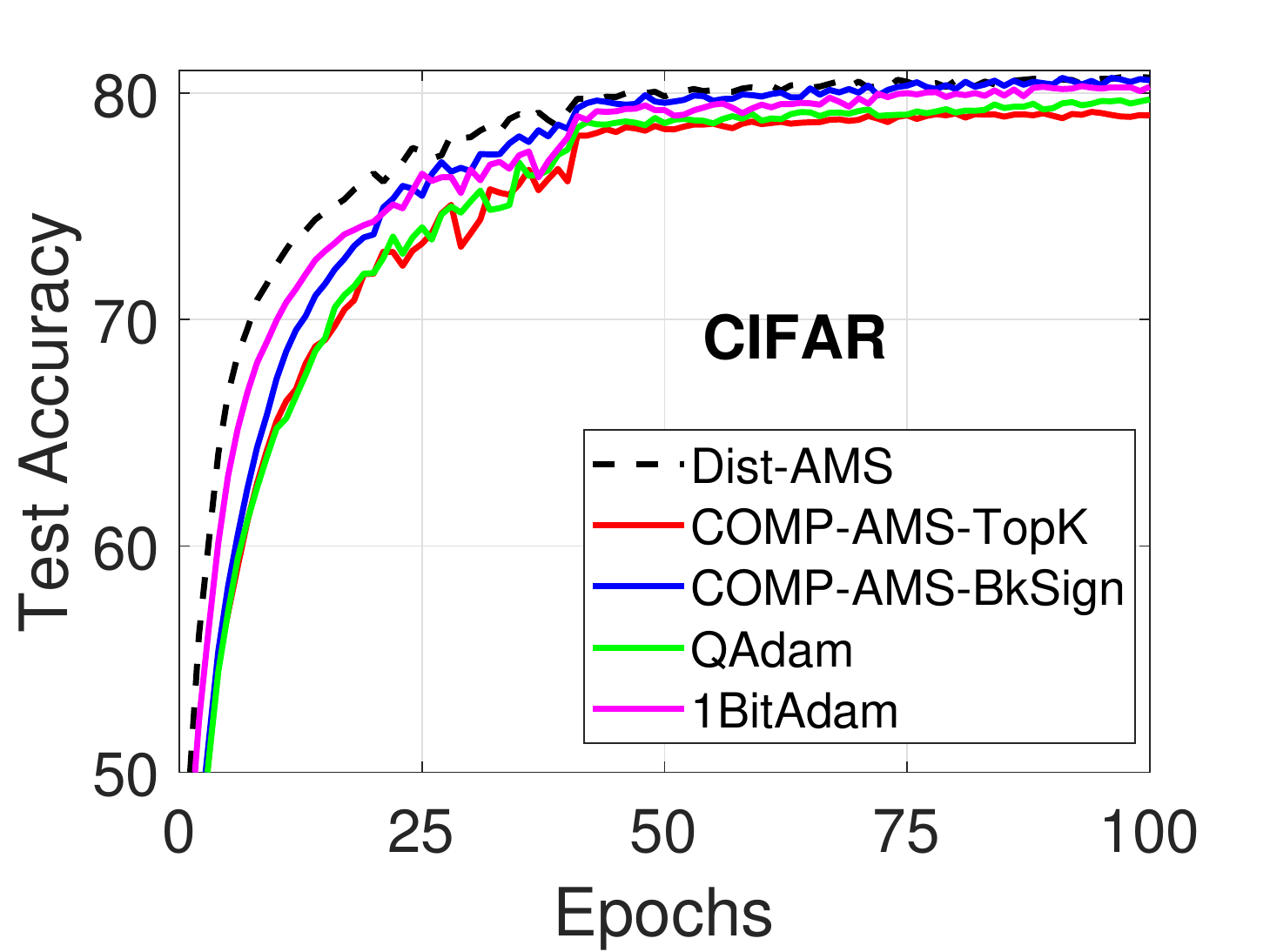}\hspace{-0.1in}
        \includegraphics[width=2in]{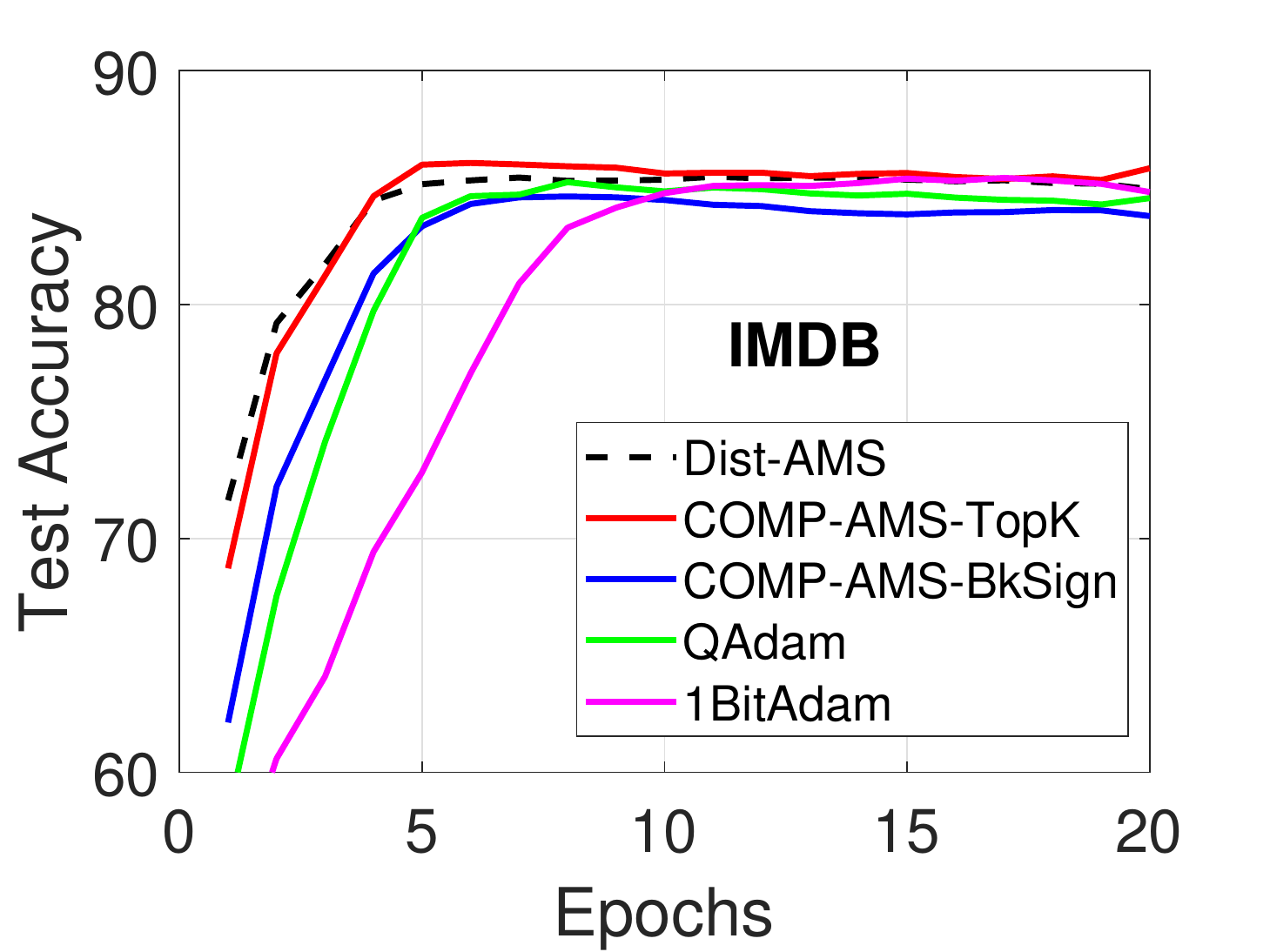}
        }
    \end{center}

	\caption{Training loss and test accuracy vs. epochs, on MNIST + CNN, CIFAR-10 + LeNet and IMDB + LSTM with $n=16$ local workers.}
	\label{fig:loss acc}
\end{figure}

\begin{figure}[h]
    \begin{center}
        \mbox{\hspace{-0.2in}
        \includegraphics[width=2in]{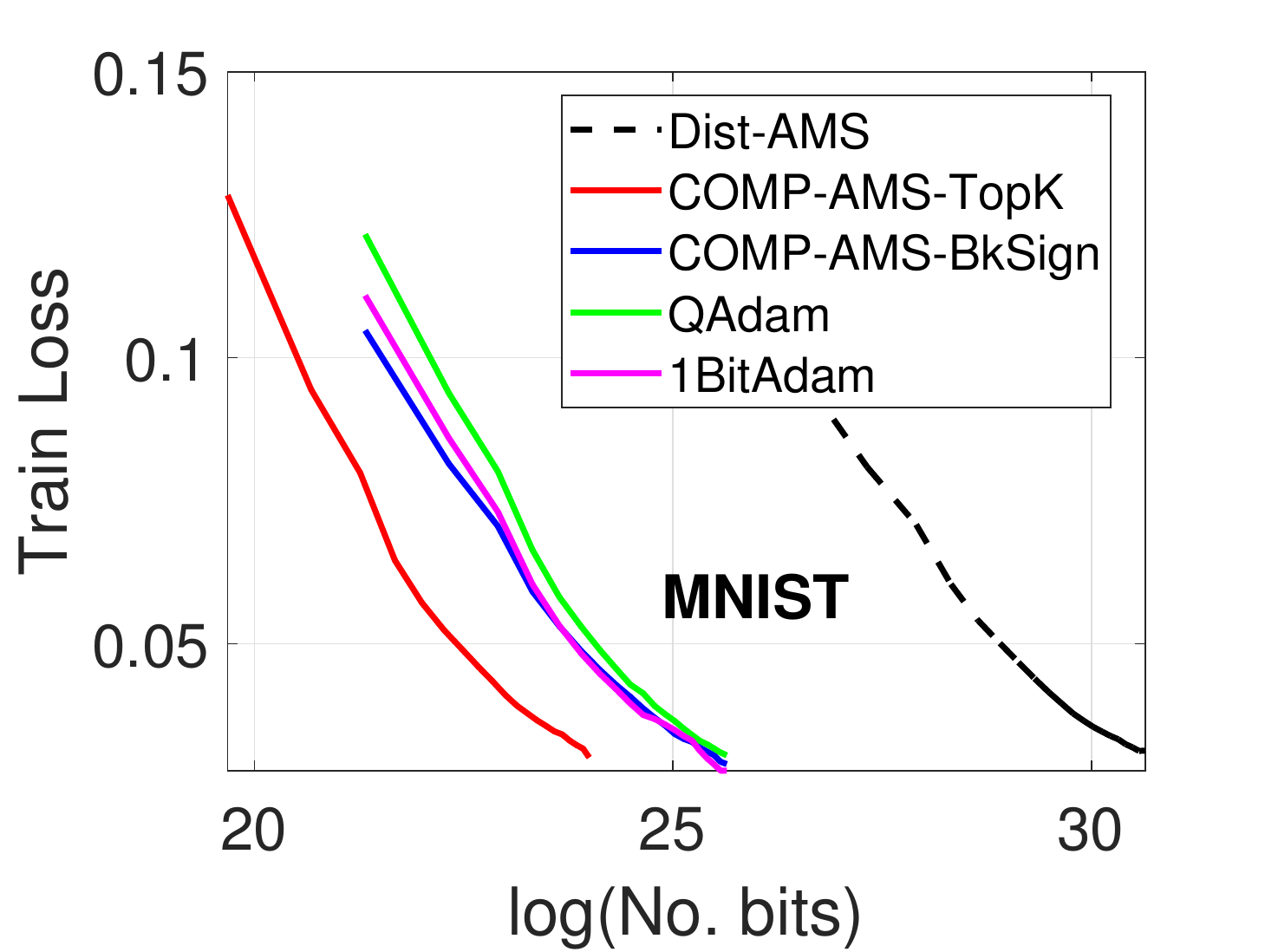}\hspace{-0.1in}
        \includegraphics[width=2in]{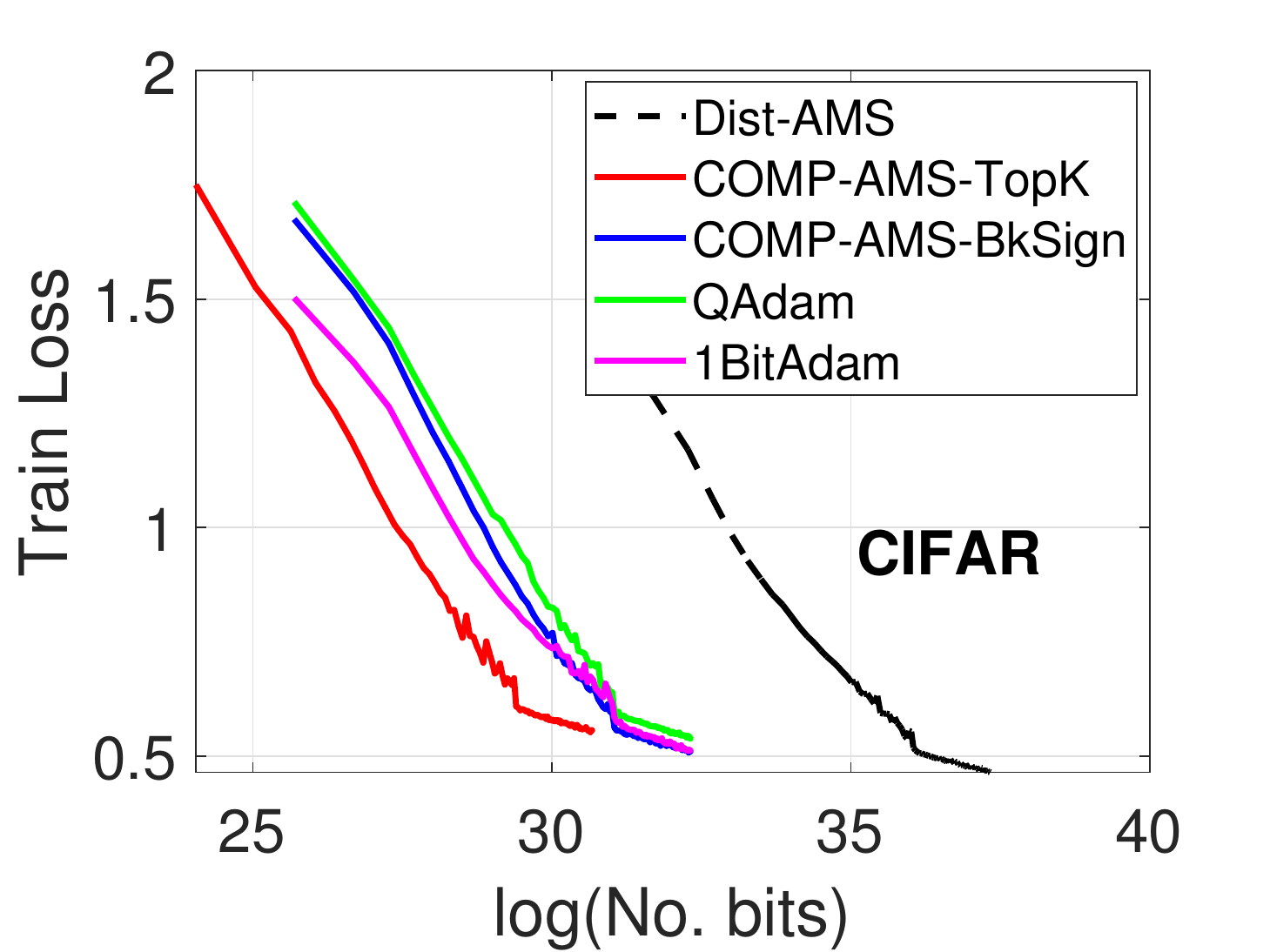}\hspace{-0.1in}
        \includegraphics[width=2in]{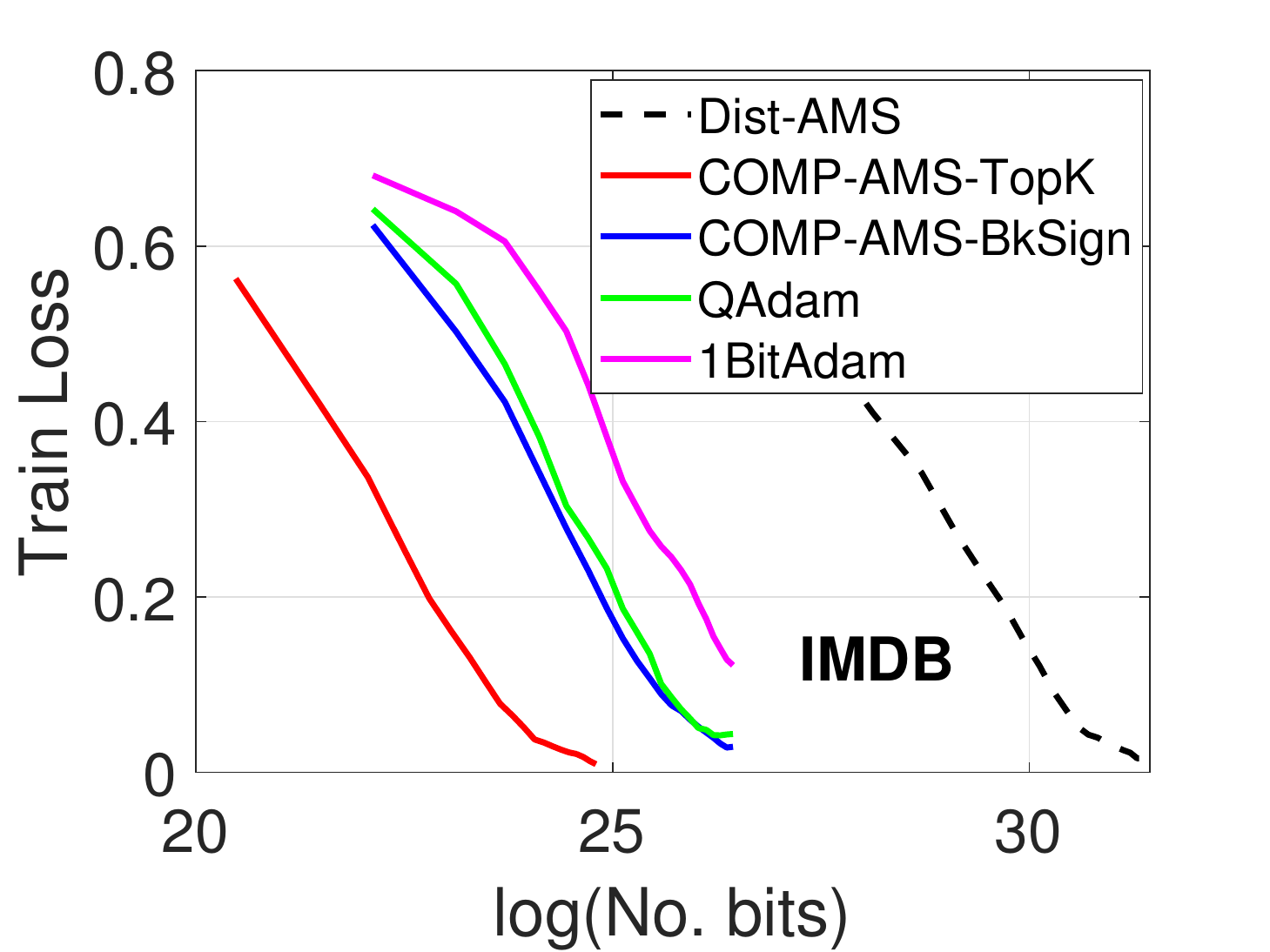}
        }
        \mbox{\hspace{-0.2in}
        \includegraphics[width=2in]{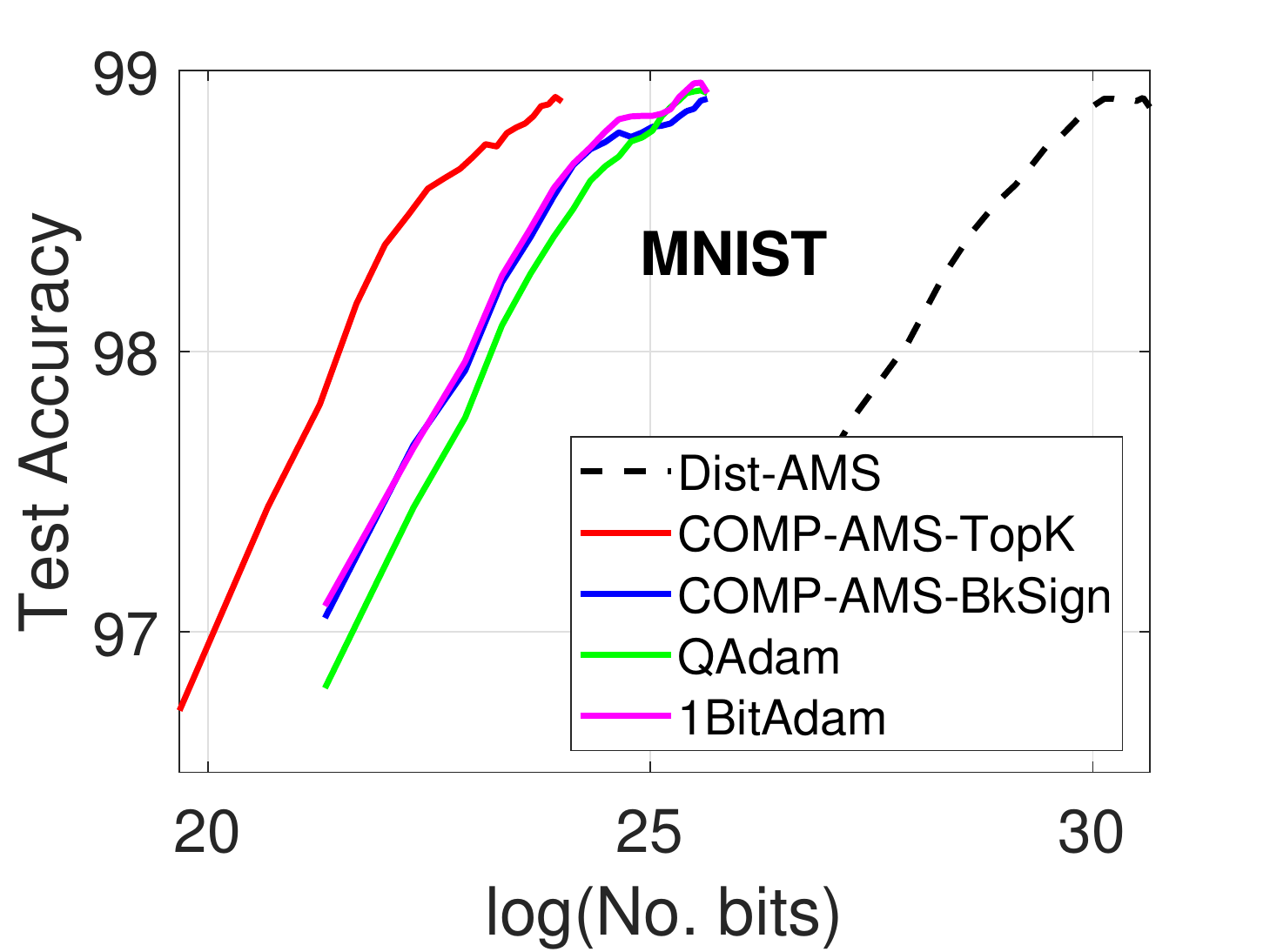}\hspace{-0.1in}
        \includegraphics[width=2in]{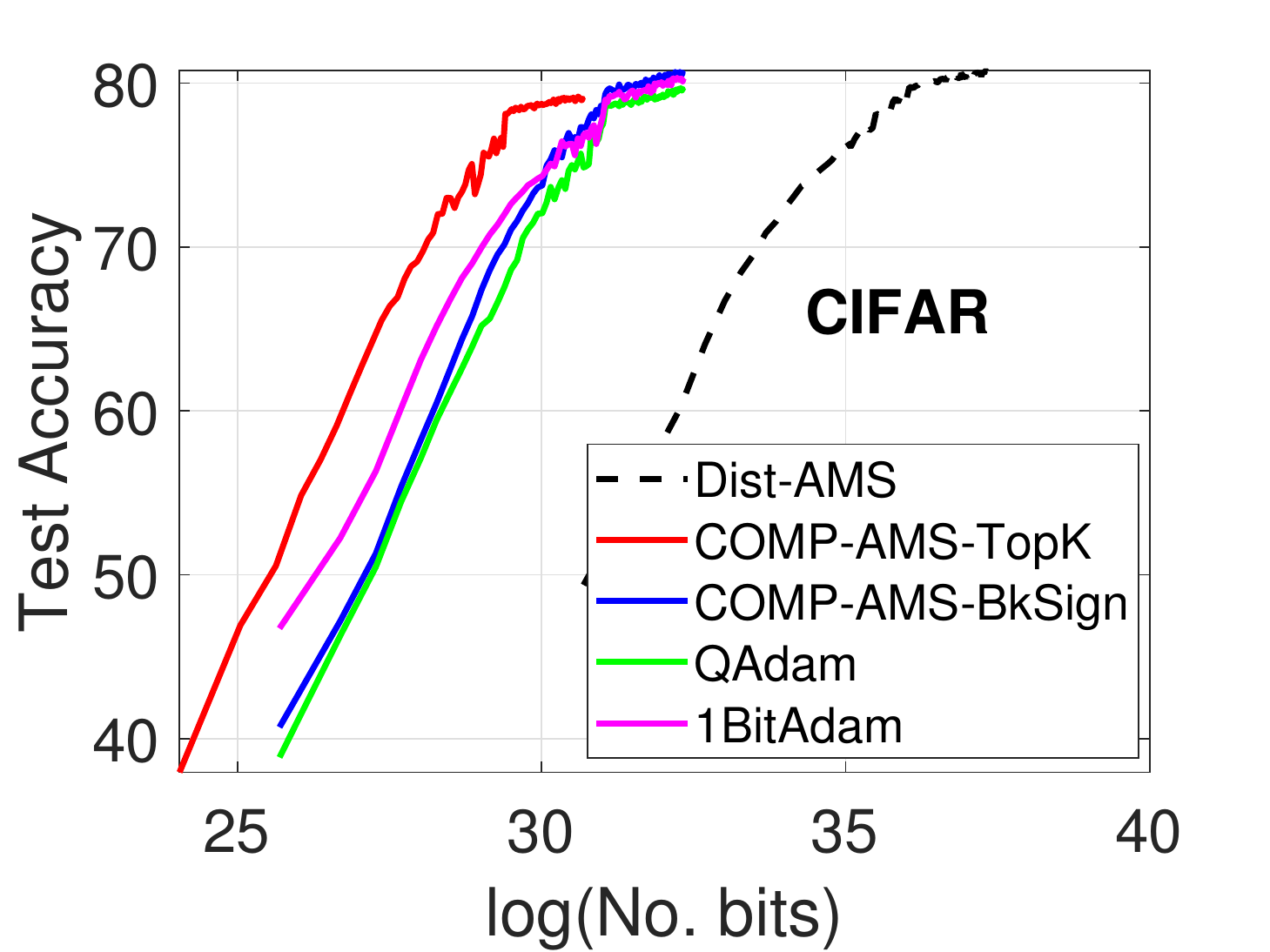}\hspace{-0.1in}
        \includegraphics[width=2in]{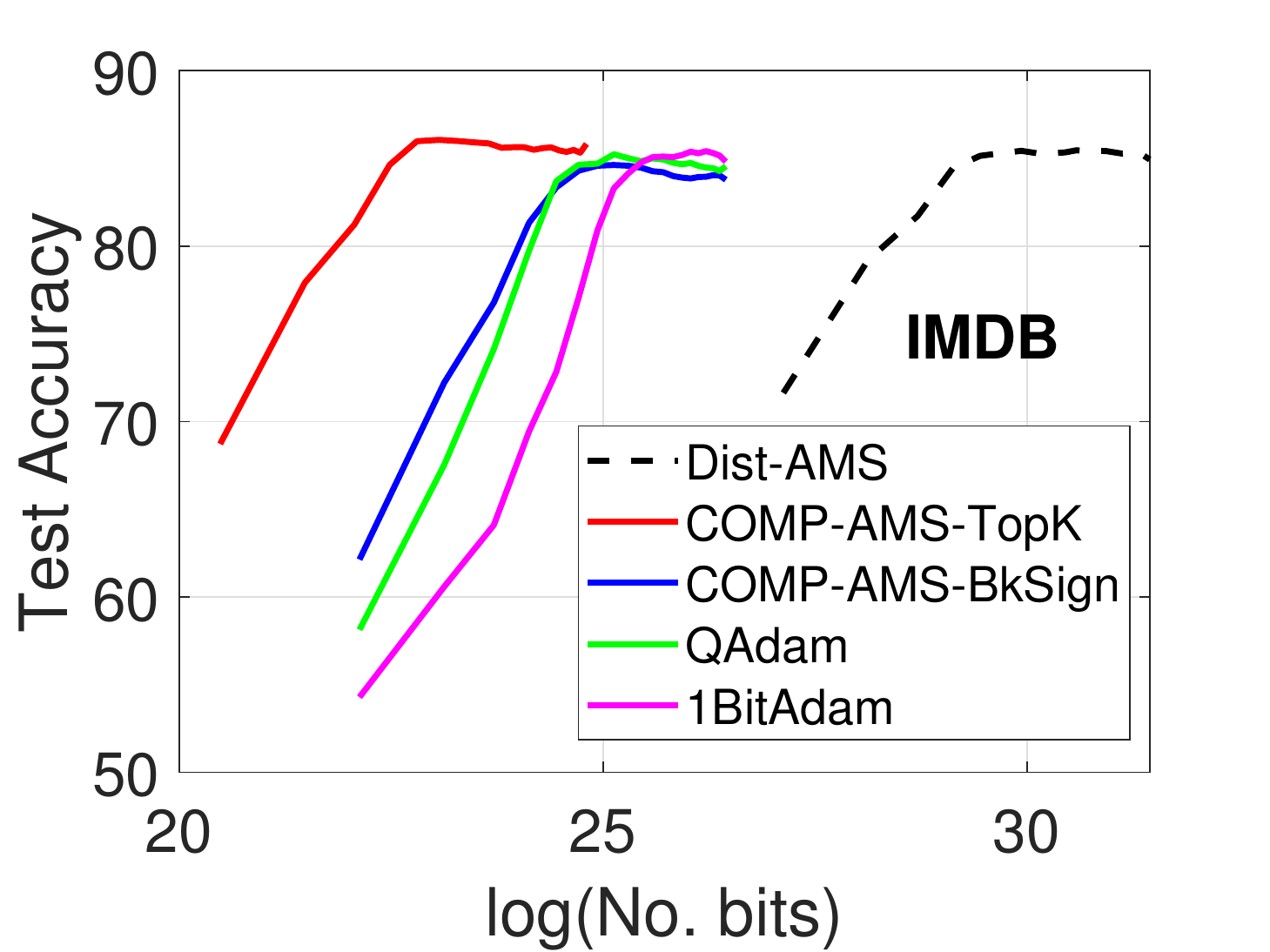}
        }
    \end{center}

	\caption{Train loss and Test accuracy vs. No. bits transmitted, on MNIST + CNN, CIFAR-10 + LeNet and IMDB + LSTM with $n=16$ local workers.}
	\label{fig:communication}
\end{figure}

\subsection{Linear Speedup of \algo}

\begin{figure}[ht]
  \begin{center}
   \mbox{
    \includegraphics[width=2.3in]{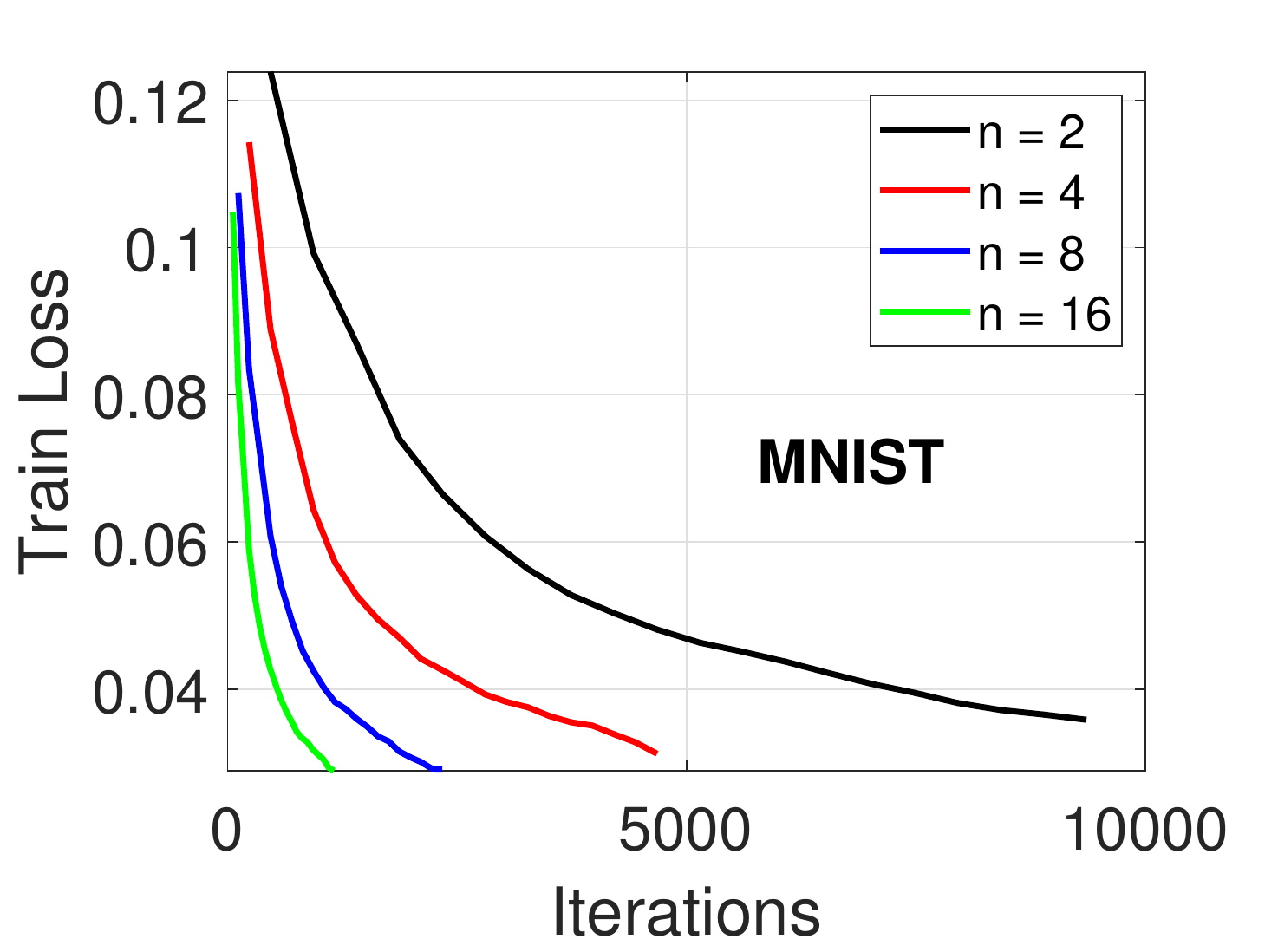}
    \includegraphics[width=2.3in]{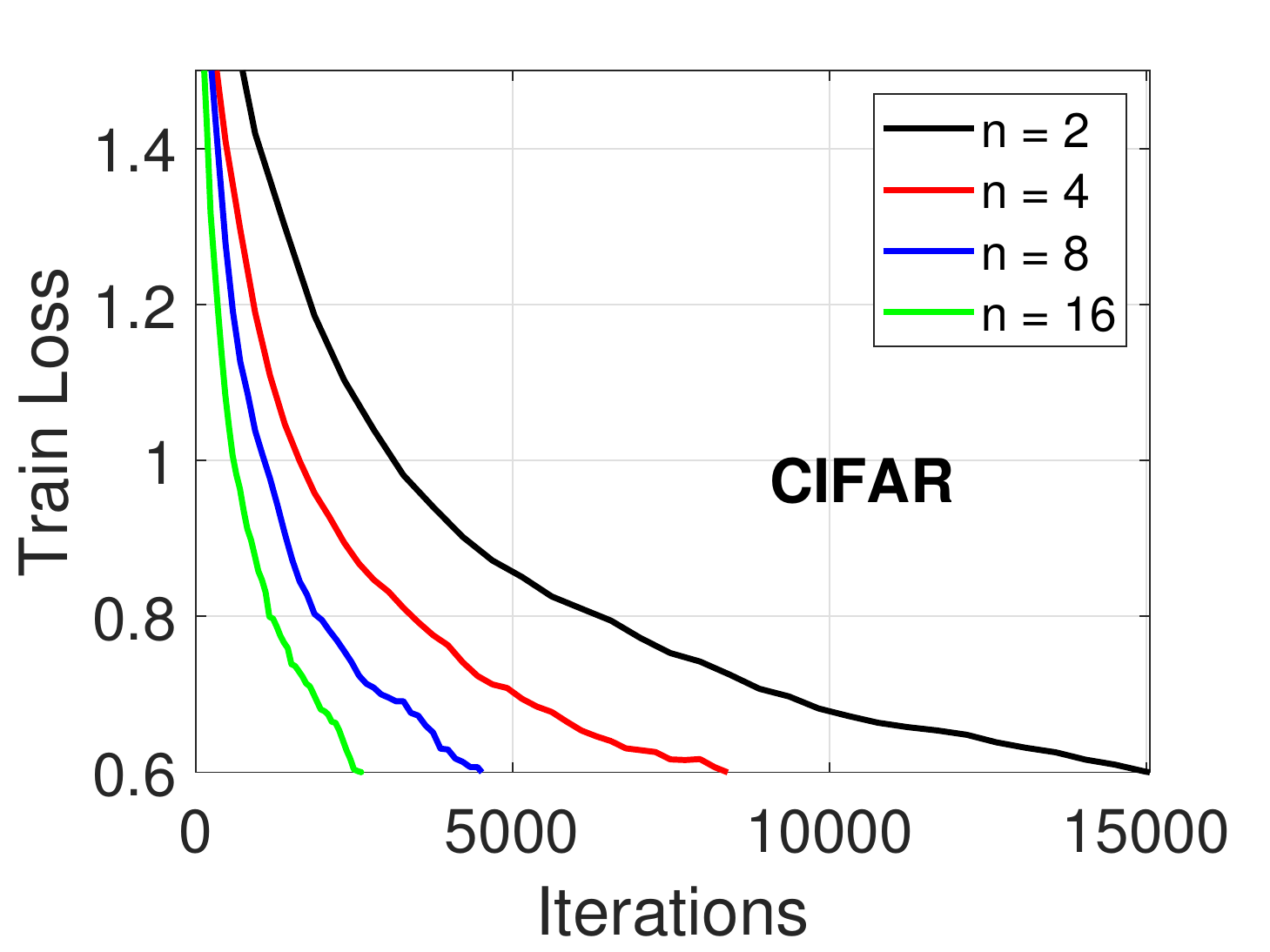}
    }
  \end{center}
  \caption{The linear speedup of \algo\ with varying $n$. \textbf{Left:} MNIST with \textbf{Block-Sign} compressor on CNN. \textbf{Right:} CIFAR-10 with \textbf{Top-$k$-0.01} compression on LeNet.}
  \label{fig:speedup}
\end{figure}

Corollary~\ref{coro:linear speedup} reveals the linear speedup of \algo\ in distributed training. In Figure~\ref{fig:speedup}, we present the training loss on MNIST and CIFAR-10 against the number of iterations, with varying number of workers $n$. We use \algo\ with \textbf{Block-Sign} on MNIST, and \textbf{Top-$k$-0.01} on CIFAR. As suggested by the theory, we use $5\times 10^{-4}\sqrt n$ as the learning rate. From Figure~\ref{fig:speedup}, we see the number of iterations to achieve a certain loss exhibits a strong linear relationship with $n$---it (approximately) decreases by half whenever we double $n$, which justifies the linear speedup of \algo.

\subsection{Discussion}

We provide a brief summary of our empirical observations. The proposed \algo\ is able to match the learning performance of full-gradient AMSGrad in all the presented experiments. In particular, for data/model involving some sparsity structure, \algo\ with the \textbf{Top-$k$} compressor could be more effective. Also, our results reveal that 1BitAdam might be quite sensitive to the pre-conditioning quality, while \algo\ can be more easily tuned and implemented in practice.

We would like to emphasize that, the primary goal of the experiments is to show that \algo\ is able to match the performance of full-precision AMSGrad, but not to argue that \algo\ is always better than the other algorithms. Since different methods use different underlying optimization algorithms (e.g., AMSGrad, Adam, momentum SGD), comparing \algo\ with other distributed training methods would be largely determined by the comparison among these optimization protocols, which is typically data and task dependent. Our results say that: whenever one wants to use AMSGrad to train a deep neural network, she/he can simply employ the distributed \algo\ scheme to gain a linear speedup in training time with learning performance as good as the full-precision training, taking little communication cost and memory consumption.

\section{Conclusion}\label{sec:conclusion}

In this paper, we study the simple, convenient, yet unexplored gradient averaging strategy for distributed adaptive optimization called \algo. \textbf{Top-$k$} and \textbf{Block-Sign} compressor are incorporated for communication efficiency, whose biases are compensated by the error feedback strategy. We develop the convergence rate of \algo, and show that same as the case of SGD, for AMSGrad, compressed gradient averaging with error feedback matches the convergence of full-gradient AMSGrad, and linear speedup can be obtained in the distributed training. Numerical experiments are conducted to justify the theoretical findings, and demonstrate that \algo\ provides comparable performance with other distributed adaptive methods, and achieves similar accuracy as full-precision AMSGrad with significantly reduced communication overhead. Given the simple architecture and hardware (memory) efficiency, we expect \algo\ shall be able to serve as a useful and convenient distributed adaptive optimization framework in practice.

\bibliographystyle{iclr2022_conference}
\bibliography{ref}
\clearpage
\newpage

\appendix

\section{Tuning Details and More Results on ResNet-18} \label{app sec:resnet}

The search grids of the learning rate of each method can be found in Table~\ref{tab:tuning}. Empirically, Dist-AMS, \algo\ and 1BitAdam has similar optimal learning rate, while QAdam usually needs larger step size to reach its best performance.

\vspace{0.1in}

\begin{table}[h]
\centering
\caption{Search grids for learning rate tuning.}
\label{tab:tuning}
\begin{tabular}{c|c}
\toprule[1pt]
 & Learning rate range     \\ \hline
Dist-AMS                  & $[0.00001,0.00003,0.00005,0.0001,0.0003,0.0005,0.001,0.003,0.005,0.01]$                      \\\hline
Comp-AMS                  & $[0.00001,0.00003,0.00005,0.0001,0.0003,0.0005,0.001,0.003,0.005,0.01]$ \\\hline
QAdam                 & $[0.0001,0.0003,0.0005,0.001,0.003,0.005,0.01,0.03,0.05,0.1,0.3,0.5]$ \\\hline
1BitAdam     & $[0.00001,0.00003,0.00005,0.0001,0.0003,0.0005,0.001,0.003,0.005,0.01]$    \\
\toprule[1pt]
\end{tabular}
\end{table}

\vspace{0.2in}

We provide more experimental results on CIFAR-10 dataset, trained with ResNet-18~\citep{Proc:Resnet_CVPR16}. For reference, we also present the result of distributed SGD. As we can see from Figure~\ref{fig:cifar resnet}, again \algo\ can achieve similar accuracy as AMSGrad, and the \textbf{Top-$k$} compressor gives the best accuracy, with substantial communication reduction. Note that distributed SGD converges faster than adaptive methods, but the generalization error is slightly worse. This experiment again confirms that \algo\ can serve as a simple and convenient distributed adaptive training framework with fast convergence, reduced communication and little performance drop.

\begin{figure}[h]
    \begin{center}
        \mbox{\hspace{-0.15in}
        \includegraphics[width=2in]{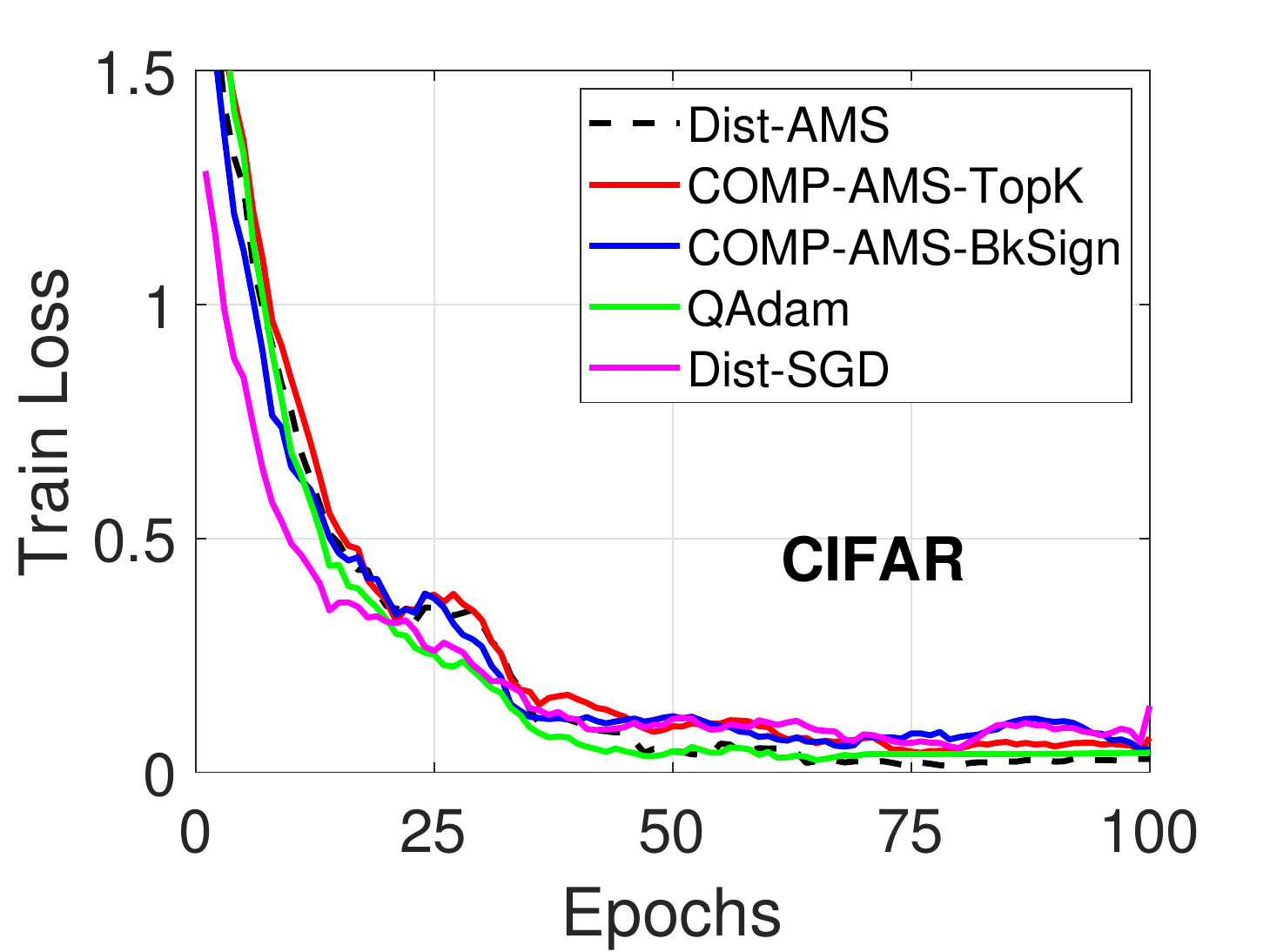}\hspace{-0.1in}
        \includegraphics[width=2in]{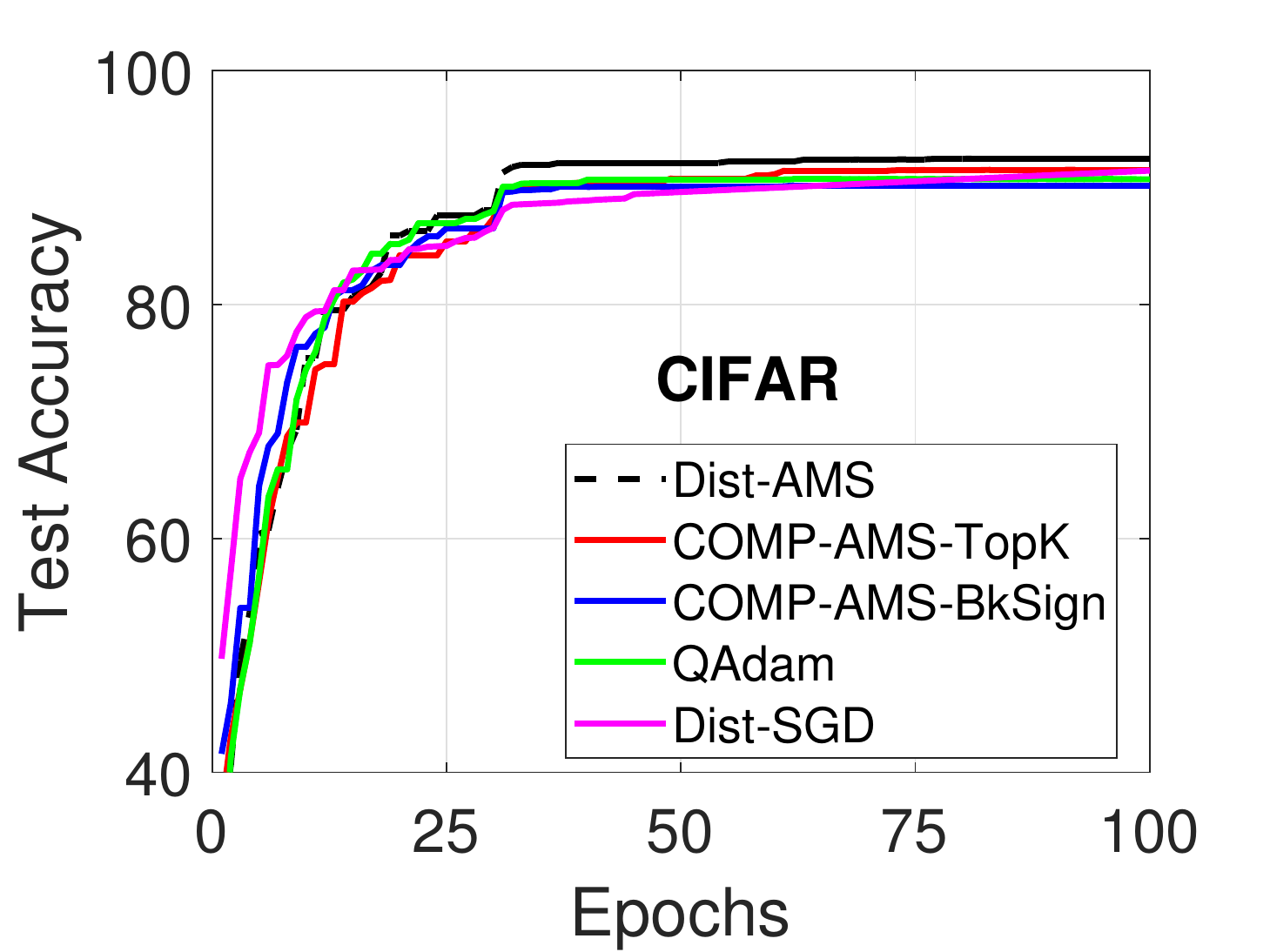}\hspace{-0.1in}
        \includegraphics[width=2in]{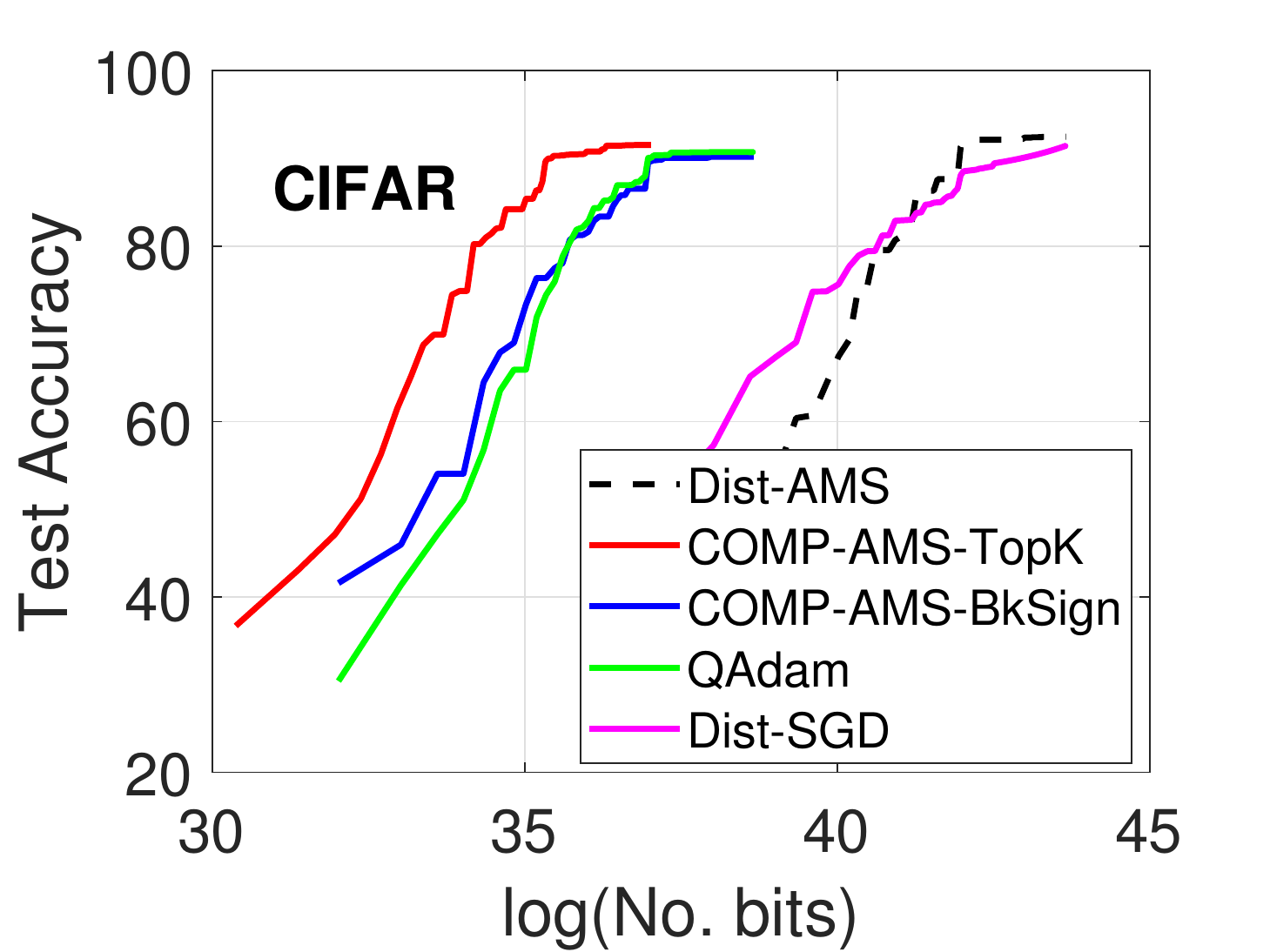}
        }
    \end{center}

	\caption{Training loss and test accuracy of different distributed training methods on CIFAR-10 with ResNet-18~\citep{Proc:Resnet_CVPR16}.}
	\label{fig:cifar resnet}
\end{figure}

\newpage\clearpage

\section{Proof of Convergence Results}\label{app:proof}

In this section, we provide the proof of our main result.

\subsection{Proof of Theorem~\ref{theo:rate}}\label{app:thm}

\begin{Theorem*}
Denote $C_0=\sqrt{\frac{4(1+q^2)^3}{(1-q^2)^2}G^2+\epsilon}$, $C_1=\frac{\beta_1}{1-\beta_1}+\frac{2q}{1-q^2}$. Under Assumption~\ref{ass:quant} to Assumption~\ref{ass:var}, with $\eta_t=\eta\leq \frac{\epsilon}{3C_0\sqrt{2L \max\{2L,C_2\}}}$, for any $T >0$, \algo\ satisfies
\begin{align*}
    \frac{1}{T}\sum_{t=1}^T \mathbb E[\|\nabla f(\theta_t)\|^2]
    &\leq 2C_0\Big(\frac{\mathbb E[f(\theta_1)-f(\theta^*)]}{T\eta}+\frac{\eta L \sigma^2}{n\epsilon}+\frac{3\eta^2 LC_0C_1^2\sigma^2}{n\epsilon^2}  \\
    &\hspace{0.7in} + \frac{12\eta^2q^2LC_0\sigma_g^2}{(1-q^2)^2\epsilon^2}+\frac{ (1+C_1)G^2d}{T\sqrt\epsilon}+\frac{\eta (1+2C_1)C_1LG^2d}{T\epsilon} \Big).
\end{align*}
\end{Theorem*}

\begin{proof}
We first clarify some notations. At time $t$, let the full-precision gradient of the $i$-th worker be $g_{t,i}$, the error accumulator be $e_{t,i}$, and the compressed gradient be $\tilde g_{t,i}=\mathcal C(g_{t,i}+e_{t,i})$. Slightly different from the notations in the algorithm, we denote $\bar g_t=\frac{1}{n}\sum_{i=1}^n g_{t,i}$, $\overline{\tilde g}_t=\frac{1}{n}\sum_{i=1}^n \tilde g_{t,i}$ and $\bar e_t=\frac{1}{n}\sum_{i=1}^n e_{t,i}$. The second moment computed by the compressed gradients is denoted as $v_t=\beta_2 v_{t-1}+(1-\beta_2) \overline{\tilde g}_t^2$, and $\hat v_t=\max\{\hat v_{t-1}, v_t\}$. Also, the first order moving average sequence
\begin{align*}
m_t=\beta_1 m_{t-1}+(1-\beta_1)\overline{\tilde g}_t \quad & \textrm{and} \quad m_t'=\beta_1 m_{t-1}'+(1-\beta_1) \bar g_t,
\end{align*}
where $m_t'$ represents the first moment moving average sequence using the uncompressed stochastic gradients. By construction we have $m_t'=(1-\beta_1)\sum_{\tau=1}^t \beta_1^{t-\tau} \bar g_\tau$.

Our proof will use the following auxiliary sequences,
\begin{align*}
& \mathcal E_{t+1}\eqdef (1-\beta_1)\sum_{\tau=1}^{t+1} \beta_1^{t+1-\tau} \bar e_\tau,\\
&\theta_{t+1}':=\theta_{t+1}-\eta \frac{\mathcal E_{t+1}}{\sqrt{\hat v_t+\epsilon}}.
\end{align*}

Then, we can write the evolution of $\theta_t'$ as
\begin{align*}
    \theta_{t+1}'&=\theta_{t+1}-\eta \frac{\mathcal E_{t+1}}{\sqrt{\hat v_t+\epsilon}}\\
    &=\theta_t-\eta\frac{(1-\beta_1)\sum_{\tau=1}^{t} \beta_1^{t-\tau}\overline{\tilde g}_\tau+(1-\beta_1)\sum_{\tau=1}^{t+1} \beta_1^{t+1-\tau}\bar e_\tau}{\sqrt{\hat v_t+\epsilon}}\\
    &=\theta_t-\eta\frac{(1-\beta_1)\sum_{\tau=1}^{t} \beta_1^{t-\tau}(\overline{\tilde g}_\tau+ \bar e_{\tau+1})+(1-\beta)\beta_1^t \bar e_1}{\sqrt{\hat v_t+\epsilon}}\\
    &=\theta_t-\eta\frac{(1-\beta_1)\sum_{\tau=1}^{t} \beta_1^{t-\tau} \bar e_\tau}{\sqrt{\hat v_t+\epsilon}}-\eta\frac{m_t'}{\sqrt{\hat v_t+\epsilon}}\\
    &=\theta_t-\eta\frac{\mathcal E_t}{\sqrt{\hat v_{t-1}+\epsilon}}-\eta\frac{m_t'}{\sqrt{\hat v_t+\epsilon}}+\eta(\frac{1}{\sqrt{\hat v_{t-1}+\epsilon}}-\frac{1}{\sqrt{\hat v_t+\epsilon}})\mathcal E_t\\
    &\overset{(a)}{=}\theta_t'-\eta\frac{m_t'}{\sqrt{\hat v_t+\epsilon}}+\eta(\frac{1}{\sqrt{\hat v_{t-1}+\epsilon}}-\frac{1}{\sqrt{\hat v_t+\epsilon}})\mathcal E_t\\
    &\eqdef \theta_t'-\eta \frac{m_t'}{\sqrt{\hat v_t+\epsilon}}+\eta D_t\mathcal E_t,
\end{align*}
where (a) uses the fact that for every $i\in[n]$, $\tilde g_{t,i}+e_{{t+1,i}}=g_{t,i}+e_{t,i}$, and $e_{t,1}=0$ at initialization. Further define the virtual iterates:
\begin{align*}
    x_{t+1}\eqdef\theta_{t+1}'-\eta \frac{\beta_1}{1-\beta_1} \frac{m_t'}{\sqrt{\hat v_t+\epsilon}},
\end{align*}
which follows the recurrence:
\begin{align*}
    x_{t+1}&=\theta_{t+1}'-\eta\frac{\beta_1}{1-\beta_1} \frac{m_t'}{\sqrt{\hat v_t+\epsilon}}\\
    &=\theta_t'-\eta\frac{m_t'}{\sqrt{\hat v_t+\epsilon}}-\eta\frac{\beta_1}{1-\beta_1} \frac{m_t'}{\sqrt{\hat v_t+\epsilon}}+\eta D_t\mathcal E_t\\
    &=\theta_t'-\eta \frac{\beta_1 m_{t-1}'+(1-\beta_1)\bar g_t+\frac{\beta_1^2}{1-\beta_1}m_{t-1}'+\beta_1 \bar g_t}{\sqrt{\hat v_t+\epsilon}}+\eta D_t\mathcal E_t\\
    &=\theta_t'-\eta\frac{\beta_1}{1-\beta_1}\frac{m_{t-1}'}{\sqrt{\hat v_t+\epsilon}}-\eta\frac{\bar g_t}{\sqrt{\hat v_t+\epsilon}}+\eta D_t\mathcal E_t\\
    &=x_t-\eta\frac{\bar g_t}{\sqrt{\hat v_t+\epsilon}}+\eta\frac{\beta_1}{1-\beta_1} D_t m_{t-1}'+\eta D_t\mathcal E_t.
\end{align*}

When summing over $t=1,...,T$, the difference sequence $D_t$ satisfies the bounds of Lemma~\ref{lemma:bound difference}.

By the smoothness Assumption~\ref{ass:smooth}, we have
\begin{align*}
    f(x_{t+1})\leq f(x_t)+\langle \nabla f(x_t), x_{t+1}-x_t\rangle+\frac{L}{2}\| x_{t+1}-x_t\|^2.
\end{align*}
Taking expectation w.r.t. the randomness at time $t$, we obtain
\begin{align}
    &\mathbb E[f(x_{t+1})]-f(x_t) \nonumber\\
    &\leq -\eta\mathbb E[\langle \nabla f(x_t), \frac{\bar g_t}{\sqrt{\hat v_t+\epsilon}}\rangle]+\eta \mathbb E[\langle \nabla f(x_t), \frac{\beta_1}{1-\beta_1}D_tm_{t-1}'+D_t\mathcal E_t\rangle] \nonumber\\
    &\hspace{2in} +\frac{\eta^2L}{2}\mathbb E[\|\frac{\bar g_t}{\sqrt{\hat v_t+\epsilon}}-\frac{\beta_1}{1-\beta_1}D_tm_{t-1}'- D_t\mathcal E_t\|^2] \nonumber\\
    &=\underbrace{-\eta\mathbb E[\langle \nabla f(\theta_t), \frac{\bar g_t}{\sqrt{\hat v_t+\epsilon}}\rangle]}_{I}+\underbrace{\eta \mathbb E[\langle \nabla f(x_t), \frac{\beta_1}{1-\beta_1}D_tm_{t-1}'+D_t\mathcal E_t\rangle]}_{II} \nonumber\\
    &\hspace{0.5in} +\underbrace{\frac{\eta^2L}{2}\mathbb E[\|\frac{\bar g_t}{\sqrt{\hat v_t+\epsilon}}-\frac{\beta_1}{1-\beta_1}D_tm_{t-1}'- D_t\mathcal E_t\|^2]}_{III}+\underbrace{\eta\mathbb E[\langle \nabla f(\theta_t)-\nabla f(x_t), \frac{\bar g_t}{\sqrt{\hat v_t+\epsilon}} \rangle]}_{IV}, \label{eq0}
\end{align}

In the following, we bound the terms separately.

\textbf{Bounding term I.} We have
\begin{align}
    I&=-\eta\mathbb E[\langle \nabla f(\theta_t), \frac{\bar g_t}{\sqrt{\hat v_{t-1}+\epsilon}}]-\eta\mathbb E[\langle \nabla f(\theta_t), (\frac{1}{\sqrt{\hat v_t+\epsilon}}-\frac{1}{\sqrt{\hat v_{t-1}+\epsilon}})\bar g_t\rangle] \nonumber\\
    &\leq -\eta\mathbb E[\langle \nabla f(\theta_t), \frac{\nabla f(\theta_t)}{\sqrt{\hat v_{t-1}+\epsilon}}]+\eta G^2\mathbb E[\|D_t\|].  \nonumber\\
    &\leq -\frac{\eta}{\sqrt{\frac{4(1+q^2)^3}{(1-q^2)^2}G^2+\epsilon}}\mathbb E[\|\nabla f(\theta_t)\|^2]+\eta G^2\mathbb E[\|D_t\|_1], \label{eq:I}
\end{align}
where we use Assumption~\ref{ass:boundgrad}, Lemma~\ref{lemma:bound v_t} and the fact that $l_2$ norm is no larger than $l_1$ norm.

\textbf{Bounding term II.} By the definition of $\mathcal E_t$, we know that $\|\mathcal E_t\|\leq (1-\beta_1)\sum_{\tau=1}^t \beta_1^{t-\tau}\|\bar e_t\|\leq \frac{2q}{1-q^2}G$. Then we have
\begin{align}
    II&\leq\eta(\mathbb E[\langle  \nabla f(\theta_t),\frac{\beta_1}{1-\beta_1}D_tm_{t-1}'+D_t\mathcal E_t\rangle]+\mathbb E[\langle  \nabla f(x_t)-\nabla f(\theta_t),\frac{\beta_1}{1-\beta_1}D_tm_{t-1}'+D_t\mathcal E_t\rangle]) \nonumber\\
    &\leq \eta\mathbb E[\|\nabla f(\theta_t)\|\|\frac{\beta_1}{1-\beta_1}D_tm_{t-1}'+D_t\mathcal E_t\|]+\eta^2 \ L \mathbb E[\|\frac{\frac{\beta_1}{1-\beta_1}m_{t-1}'+\mathcal E_t}{\sqrt{\hat v_{t-1}+\epsilon}}\| \|\frac{\beta_1}{1-\beta_1}D_tm_{t-1}'+D_t\mathcal E_t\|] \nonumber\\
    &\leq \eta C_1 G^2 \mathbb E[\|D_t\|_1]+\frac{\eta^2 C_1^2 LG^2}{\sqrt\epsilon}\mathbb E[\|D_t\|_1],  \label{eq:II}
\end{align}
where $C_1\eqdef \frac{\beta_1}{1-\beta_1}+\frac{2q}{1-q^2}$. The second inequality is because of smoothness of $f(\theta)$, and the last inequality is due to Lemma~\ref{lemma:bound e_t}, Assumption~\ref{ass:boundgrad} and the property of norms.

\textbf{Bounding term III.} This term can be bounded as follows:
\begin{align}
    III&\leq \eta^2 L\mathbb E[\|\frac{\bar g_t}{\sqrt{\hat v_t+\epsilon}}\|^2]+\eta^2 L\mathbb E[\|\frac{\beta_1}{1-\beta_1}D_tm_{t-1}'- D_t\mathcal E_t\|^2]] \nonumber\\
    &\leq \frac{\eta^2 L}{\epsilon}\mathbb E[\|\frac{1}{n}\sum_{i=1}^n g_{t,i}-\nabla f(\theta_t)+\nabla f(\theta_t)\|^2]+\eta^2 L\mathbb E[\|D_t(\frac{\beta_1}{1-\beta_1}m_{t-1}'-\mathcal E_t)\|^2] \nonumber\\
    &\overset{(a)}{\leq} \frac{\eta^2 L}{\epsilon}\mathbb E[\|\nabla f(\theta_t)\|^2]+\frac{\eta^2 L \sigma^2}{n \epsilon}+\eta^2 C_1^2 LG^2 \mathbb E[\|D_t\|^2],  \label{eq:III}
\end{align}
where (a) follows from $\nabla f(\theta_t)=\frac{1}{n}\sum_{i=1}^n \nabla f_i(\theta_t)$ and Assumption~\ref{ass:var} that $g_{t,i}$ is unbiased of $\nabla f_i(\theta_t)$ and has bounded variance $\sigma^2$.

\textbf{Bounding term IV.} We have
\begin{align}
    IV&=\eta\mathbb E[\langle \nabla f(\theta_t)-\nabla f(x_t), \frac{\bar g_t}{\sqrt{\hat v_{t-1}+\epsilon}} \rangle]+\eta\mathbb E[\langle \nabla f(\theta_t)-\nabla f(x_t), (\frac{1}{\sqrt{\hat v_t+\epsilon}}-\frac{1}{\sqrt{\hat v_{t-1}+\epsilon}})\bar g_t \rangle] \nonumber\\
    &\leq \eta\mathbb E[\langle \nabla f(\theta_t)-\nabla f(x_t), \frac{\nabla f(\theta_t)}{\sqrt{\hat v_{t-1}+\epsilon}} \rangle]+\eta^2 L\mathbb E[\|\frac{\frac{\beta_1}{1-\beta_1}m_{t-1}'+\mathcal E_t}{\sqrt{\hat v_{t-1}+\epsilon}}\|\|D_t g_t\|] \nonumber\\
    &\overset{(a)}{\leq} \frac{\eta \rho}{2\epsilon}\mathbb E[\|\nabla f(\theta_t)\|^2]+\frac{\eta}{2\rho}\mathbb E[\|\nabla f(\theta_t)-\nabla f(x_t)\|^2]+\frac{\eta^2 C_1LG^2}{\sqrt\epsilon} \mathbb E[\|D_t\|]  \nonumber\\
    &\overset{(b)}{\leq} \frac{\eta \rho}{2\epsilon}\mathbb E[\|\nabla f(\theta_t)\|^2]+\frac{\eta^3 L}{2\rho}\mathbb E[\|\frac{\frac{\beta_1}{1-\beta_1}m_{t-1}'+\mathcal E_t}{\sqrt{\hat v_{t-1}+\epsilon}}\|^2]+\frac{\eta^2 C_1LG^2}{\sqrt\epsilon} \mathbb E[\|D_t\|_1],  \label{eq:IV}
\end{align}
where (a) is due to Young's inequality and (b) is based on Assumption~\ref{ass:smooth}.

Regarding the second term in \eqref{eq:IV}, by Lemma~\ref{lemma:bound big E_t} and Lemma~\ref{lemma:m_t,m_t'}, summing over $t=1,...,T$ we have
\begin{align}
    &\sum_{t=1}^T\frac{\eta^3 L}{2\rho}\mathbb E[\|\frac{\frac{\beta_1}{1-\beta_1}m_{t-1}'+\mathcal E_t}{\sqrt{\hat v_{t-1}+\epsilon}}\|^2] \nonumber\\
    &\leq \sum_{t=1}^T\frac{\eta^3 L}{2\rho\epsilon} \mathbb E[\|\frac{\beta_1}{1-\beta_1}m_{t-1}'+\mathcal E_t\|^2] \nonumber\\
    &\leq \sum_{t=1}^T\frac{\eta^3 L}{\rho\epsilon}\Big[ \frac{\beta_1^2}{(1-\beta_1)^2}\mathbb E[\|m_t'\|^2]+ \mathbb E[\|\mathcal E_t\|^2]\Big] \nonumber\\
    &\leq \frac{T\eta^3\beta_1^2 L \sigma^2}{n\rho(1-\beta_1)^2\epsilon}+\frac{\eta^3\beta_1^2 L}{\rho(1-\beta_1)^2\epsilon}\sum_{t=1}^T \mathbb E[\|\nabla f(\theta_t)\|^2] \nonumber\\
    &\hspace{2in} +\frac{4T\eta^3q^2L}{\rho(1-q^2)^2\epsilon}(\sigma^2+\sigma_g^2) + \frac{4\eta^3 q^2L}{\rho(1-q^2)^2\epsilon} \sum_{t=1}^T \mathbb E[\|\nabla f(\theta_t)\|^2 ] \nonumber\\
    &=\frac{T\eta^3 LC_2\sigma^2}{n\rho\epsilon}+\frac{4T\eta^3q^2L\sigma_g^2}{\rho(1-q^2)^2\epsilon}+\frac{\eta^3LC_2}{\rho\epsilon}\sum_{t=1}^T \mathbb E[\|\nabla f(\theta_t)\|^2 ], \label{eq:IV error}
\end{align}
with $C_2\eqdef \frac{\beta_1^2}{(1-\beta_1)^2}+\frac{4q^2}{(1-q^2)^2}$. Now integrating \eqref{eq:I}, \eqref{eq:II}, \eqref{eq:III}, \eqref{eq:IV} and \eqref{eq:IV error} into \eqref{eq0}, taking the telescoping summation over $t=1,...,T$, we obtain
\begin{align*}
    &\mathbb E[f(x_{T+1})-f(x_1)]\\
    &\leq (-\frac{\eta}{C_0}+\frac{\eta^2 L}{\epsilon}+\frac{\eta \rho}{2\epsilon}+\frac{\eta^3LC_2}{\rho\epsilon})\sum_{t=1}^T\mathbb E[\|\nabla f(\theta_t)\|^2]+\frac{T\eta^2 L \sigma^2}{n\epsilon}+\frac{T\eta^3 LC_2\sigma^2}{n\rho\epsilon}+\frac{4T\eta^3q^2L\sigma_g^2}{\rho(1-q^2)^2\epsilon}  \\
    &\hspace{0.8in} + (\eta(1+C_1)G^2+\frac{\eta^2 (1+C_1)C_1LG^2}{\sqrt\epsilon})\sum_{t=1}^T\mathbb E[\|D_t\|_1]+\eta^2C_1^2LG^2 \sum_{t=1}^T\mathbb E[\|D_t\|^2.
\end{align*}
with $C_0\eqdef \sqrt{\frac{4(1+q^2)^3}{(1-q^2)^2}G^2+\epsilon}$. Setting $\eta\leq \frac{\epsilon}{3C_0\sqrt{2L \max\{2L,C_2\}}}$ and choosing $\rho=\frac{\epsilon}{3C_0}$, we further arrive at
\begin{align*}
    &\mathbb E[f(x_{T+1})-f(x_1)]\\
    &\leq -\frac{\eta}{2C_0}\sum_{t=1}^T\mathbb E[\|\nabla f(\theta_t)\|^2]+\frac{T\eta^2 L \sigma^2}{n\epsilon}+\frac{3T\eta^3 LC_0C_2\sigma^2}{n\epsilon^2}+\frac{12T\eta^3q^2LC_0\sigma_g^2}{(1-q^2)^2\epsilon^2}  \\
    &\hspace{2.2in} + \frac{\eta (1+C_1)G^2d}{\sqrt\epsilon}+\frac{\eta^2 (1+2C_1)C_1LG^2d}{\epsilon}.
\end{align*}
where the inequality follows from Lemma~\ref{lemma:bound difference}. Re-arranging terms, we get that
\begin{align*}
    \frac{1}{T}\sum_{t=1}^T \mathbb E[\|\nabla f(\theta_t)\|^2]&\leq 2C_0\Big(\frac{\mathbb E[f(x_1)-f(x_{T+1})]}{T\eta}+\frac{\eta L \sigma^2}{n\epsilon}+\frac{3\eta^2 LC_0C_2\sigma^2}{n\epsilon^2}  \\
    &\hspace{0.7in} + \frac{12\eta^2q^2LC_0\sigma_g^2}{(1-q^2)^2\epsilon^2}+\frac{ (1+C_1)G^2d}{T\sqrt\epsilon}+\frac{\eta (1+2C_1)C_1LG^2d}{T\epsilon} \Big)\\
    &\leq 2C_0\Big(\frac{\mathbb E[f(\theta_1)-f(\theta^*)]}{T\eta}+\frac{\eta L \sigma^2}{n\epsilon}+\frac{3\eta^2 LC_0C_1^2\sigma^2}{n\epsilon^2}  \\
    &\hspace{0.7in} + \frac{12\eta^2q^2LC_0\sigma_g^2}{(1-q^2)^2\epsilon^2}+\frac{ (1+C_1)G^2d}{T\sqrt\epsilon}+\frac{\eta (1+2C_1)C_1LG^2d}{T\epsilon} \Big),
\end{align*}
where $C_0=\sqrt{\frac{4(1+q^2)^3}{(1-q^2)^2}G^2+\epsilon}$, $C_1=\frac{\beta_1}{1-\beta_1}+\frac{2q}{1-q^2}$. The last inequality is because $x_1=\theta_1$, $\theta^* \eqdef \argmin_\theta f(\theta)$ and the fact that $C_2\leq C_1^2$. This completes the proof.
\end{proof}

\subsection{Intermediate Lemmata}\label{app:lemmas}

The lemmas used in the proof of Theorem~\ref{theo:rate} are given as below.

\begin{Lemma} \label{lemma:m_t,m_t'}
Under Assumption~\ref{ass:quant} to Assumption~\ref{ass:var} we have:
\begin{align*}
    &\|m_t'\|\leq G, \quad \forall t,\\
    &\sum_{t=1}^T\mathbb E\|m_t'\|^2\leq \frac{T\sigma^2}{n}+\sum_{t=1}^T \mathbb E[\|\nabla f(\theta_t)\|^2].
\end{align*}
\end{Lemma}

\begin{proof}
For the first part, it is easy to see that by Assumption~\ref{ass:boundgrad},
\begin{align*}
    \|m_t'\|&=(1-\beta_1)\|\sum_{\tau=1}^t \beta_1^{t-\tau} \bar g_t\|\leq G.
\end{align*}
For the second claim, the expected squared norm of average stochastic gradient can be bounded by
\begin{align*}
    \mathbb E[\|\bar g_t^2\|]&=\mathbb E[\|\frac{1}{n}\sum_{i=1}^n g_{t,i}-\nabla f(\theta_t)+\nabla f(\theta_t)\|^2]\\
    &=\mathbb E[\|\frac{1}{n}\sum_{i=1}^n (g_{t,i}-\nabla f_i(\theta_t))\|^2]+\mathbb E[\|\nabla f(\theta_t)\|^2]\\
    &\leq \frac{\sigma^2}{n}+\mathbb E[\|\nabla f(\theta_t)\|^2],
\end{align*}
where we use Assumption~\ref{ass:var} that $g_{t,i}$ is unbiased with bounded variance. Let $\bar g_{t,j}$ denote the $j$-th coordinate of $\bar g_t$. By the updating rule of \algo, we have
\begin{align*}
    \mathbb E[\|m_t'\|^2]&=\mathbb E[\|(1-\beta_1)\sum_{\tau=1}^t\beta_1^{t-\tau} \bar g_\tau\|^2]\\
    &\leq (1-\beta_1)^2\sum_{j=1}^d \mathbb E[(\sum_{\tau=1}^t\beta_1^{t-\tau} \bar g_{\tau,j})^2]\\
    &\overset{(a)}{\leq} (1-\beta_1)^2\sum_{j=1}^d \mathbb E[(\sum_{\tau=1}^t\beta_1^{t-\tau})(\sum_{\tau=1}^t\beta_1^{t-\tau} \bar g_{\tau,j}^2)]\\
    &\leq (1-\beta_1)\sum_{\tau=1}^t \beta_1^{t-\tau}\mathbb E[\|\bar g_\tau\|^2]\\
    &\leq \frac{\sigma^2}{n}+(1-\beta_1)\sum_{\tau=1}^t \beta_1^{t-\tau}\mathbb E[\|\nabla f(\theta_t)\|^2],
\end{align*}
where (a) is due to Cauchy-Schwartz inequality. Summing over $t=1,...,T$, we obtain
\begin{align*}
    \sum_{t=1}^T\mathbb E\|m_t'\|^2\leq \frac{T\sigma^2}{n}+\sum_{t=1}^T \mathbb E[\|\nabla f(\theta_t)\|^2].
\end{align*}
This completes the proof.
\end{proof}

\begin{Lemma} \label{lemma:bound e_t}
Under Assumption~\ref{ass:var}, we have for $\forall t$ and each local worker $\forall i\in [n]$,
\begin{align*}
    &\|e_{t,i}\|^2\leq \frac{4q^2}{(1-q^2)^2}G^2,\\
    &\mathbb E[\|e_{t+1,i}\|^2]\leq \frac{4q^2}{(1-q^2)^2}\sigma^2 + \frac{2q^2}{1-q^2}\sum_{\tau=1}^t (\frac{1+q^2}{2})^{t-\tau} \mathbb E[\|\nabla f_i(\theta_\tau)\|^2].
\end{align*}
\end{Lemma}

\begin{proof}
We start by using Assumption~\ref{ass:quant} and Young's inequality to get
\begin{align}
    \|e_{t+1,i}\|^2&=\|g_{t,i}+e_{t,i}-\mathcal C(g_{t,i}+e_{t,i})\|^2 \nonumber\\
    &\leq q^2\|g_{t,i}+e_{t,i}\|^2 \nonumber\\
    &\leq q^2(1+\rho)\|e_{t,i}\|^2+q^2(1+\frac{1}{\rho})\|g_{t,i}\|^2 \nonumber\\
    &\leq \frac{1+q^2}{2}\|e_{t,i}\|^2 + \frac{2q^2}{1-q^2}\|g_{t,i}\|^2, \label{eq:e_t 0}
\end{align}
where \eqref{eq:e_t 0} is derived by choosing $\rho=\frac{1-q^2}{2q^2}$ and the fact that $q<1$. Now by recursion and the initialization $e_{1,i}=0$, we have
\begin{align*}
    \mathbb E[\|e_{t+1,i}\|^2]&\leq \frac{2q^2}{1-q^2} \sum_{\tau=1}^t (\frac{1+q^2}{2})^{t-\tau} \mathbb E[\|g_{\tau,i}\|^2]  \\
    &\leq \frac{4q^2}{(1-q^2)^2}\sigma^2 + \frac{2q^2}{1-q^2}\sum_{\tau=1}^t (\frac{1+q^2}{2})^{t-\tau} \mathbb E[\|\nabla f_i(\theta_{\tau})\|^2], \nonumber
\end{align*}
which proves the second argument. Meanwhile, the absolute bound $\|e_{t,i}\|^2\leq \frac{4q^2}{(1-q^2)^2}G^2$ follows directly from \eqref{eq:e_t 0}.
\end{proof}

\begin{Lemma} \label{lemma:bound big E_t}
For the moving average error sequence $\mathcal E_t$, it holds that
\begin{align*}
    \sum_{t=1}^T \mathbb E[\|\mathcal E_t\|^2]\leq \frac{4Tq^2}{(1-q^2)^2}(\sigma^2+\sigma_g^2) + \frac{4q^2}{(1-q^2)^2} \sum_{t=1}^T \mathbb E[\|\nabla f(\theta_t)\|^2 ].
\end{align*}
\end{Lemma}

\begin{proof}
Denote $K_{t,i}\eqdef \sum_{\tau=1}^t (\frac{1+q^2}{2})^{t-\tau} \mathbb E[\|\nabla f_i(\theta_\tau)\|^2]$. Using the same technique as in the proof of Lemma~\ref{lemma:m_t,m_t'}, denoting $\bar e_{t,j}$ as the $j$-th coordinate of $\bar e_{t}$, it follows that
\begin{align*}
    \mathbb E[\|\mathcal E_t\|^2]&=\mathbb E[\|(1-\beta_1)\sum_{\tau=1}^t\beta_1^{t-\tau} \bar e_\tau\|^2]\\
    &\leq (1-\beta_1)^2\sum_{j=1}^d \mathbb E[(\sum_{\tau=1}^t\beta_1^{t-\tau} \bar e_{\tau,j})^2]\\
    &\overset{(a)}{\leq} (1-\beta_1)^2\sum_{j=1}^d \mathbb E[(\sum_{\tau=1}^t\beta_1^{t-\tau})(\sum_{\tau=1}^t\beta_1^{t-\tau} \bar e_{\tau,j}^2)]\\
    &\leq (1-\beta_1)\sum_{\tau=1}^t \beta_1^{t-\tau}\mathbb E[\|\bar e_\tau\|^2]\\
    &\leq (1-\beta_1)\sum_{\tau=1}^t \beta_1^{t-\tau}\mathbb E[\frac{1}{n}\sum_{i=1}^n\|e_{\tau,i}\|^2] \\
    &\overset{(b)}{\leq} \frac{4q^2}{(1-q^2)^2}\sigma^2+\frac{2q^2(1-\beta_1)}{(1-q^2)}\sum_{\tau=1}^t \beta_1^{t-\tau} (\frac{1}{n}\sum_{i=1}^n K_{\tau,i}),
\end{align*}
where (a) is due to Cauchy-Schwartz and (b) is a result of Lemma~\ref{lemma:bound e_t}. Summing over $t=1,...,T$ and using the technique of geometric series summation leads to
\begin{align*}
    \sum_{t=1}^T \mathbb E[\|\mathcal E_t\|^2]&\leq \frac{4Tq^2}{(1-q^2)^2}\sigma^2 + \frac{2q^2(1-\beta_1)}{(1-q^2)}\sum_{t=1}^T \sum_{\tau=1}^t \beta_1^{t-\tau} (\frac{1}{n}\sum_{i=1}^n K_{\tau,i})\\
    &\leq \frac{4Tq^2}{(1-q^2)^2}\sigma^2 +\frac{2q^2}{(1-q^2)}\sum_{t=1}^T\sum_{\tau=1}^t (\frac{1+q^2}{2})^{t-\tau} \mathbb E[\frac{1}{n}\sum_{i=1}^n\|\nabla f_i(\theta_\tau)\|^2]\\
    &\leq \frac{4Tq^2}{(1-q^2)^2}\sigma^2 + \frac{4q^2}{(1-q^2)^2} \sum_{t=1}^T \mathbb E[\frac{1}{n}\sum_{i=1}^n\|\nabla f_i(\theta_t)\|^2]\\
    &\overset{(a)}{\leq } \frac{4Tq^2}{(1-q^2)^2}\sigma^2 + \frac{4q^2}{(1-q^2)^2} \sum_{t=1}^T \mathbb E[\|\frac{1}{n}\sum_{i=1}^n\nabla f_i(\theta_t)\|^2+\frac{1}{n}\sum_{i=1}^n\|\nabla f_i(\theta_t)-\nabla f(\theta_t)  \|^2 ]\\
    &\leq \frac{4Tq^2}{(1-q^2)^2}(\sigma^2+\sigma_g^2) + \frac{4q^2}{(1-q^2)^2} \sum_{t=1}^T \mathbb E[\|\nabla f(\theta_t)\|^2 ],
\end{align*}
where (a) is derived by the variance decomposition and the last inequality holds due to Assumption~\ref{ass:var}. The desired result is obtained.
\end{proof}

\begin{Lemma} \label{lemma:bound v_t}
It holds that $\forall t\in [T]$, $\forall i\in [d]$, $\hat v_{t,i}\leq \frac{4(1+q^2)^3}{(1-q^2)^2}G^2$.

\end{Lemma}

\begin{proof}
For any $t$, by Lemma~\ref{lemma:bound e_t} and Assumption~\ref{ass:boundgrad} we have
\begin{align*}
    \|\tilde g_t\|^2&=\|\mathcal C(g_t+e_t)\|^2\\
    &\leq \|\mathcal C(g_t+e_t)-(g_t+e_t)+(g_t+e_t)\|^2\\
    &\leq 2(q^2+1)\|g_t+e_t\|^2\\
    &\leq 4(q^2+1)(G^2+\frac{4q^2}{(1-q^2)^2}G^2)\\
    &=\frac{4(1+q^2)^3}{(1-q^2)^2}G^2.
\end{align*}
It's then easy to show by the updating rule of $\hat v_t$, there exists a $j\in[t]$ such that $\hat v_t=v_j$. Then
\begin{align*}
    \hat v_{t,i}=(1-\beta_2)\sum_{\tau=1}^j \beta_2^{j-\tau} \tilde g_{\tau,i}^2\leq \frac{4(1+q^2)^3}{(1-q^2)^2}G^2,
\end{align*}
which concludes the claim.
\end{proof}

\begin{Lemma}  \label{lemma:bound difference}
Let $D_t\eqdef \frac{1}{\sqrt{\hat v_{t-1}+\epsilon}}-\frac{1}{\sqrt{\hat v_t+\epsilon}}$ be defined as above. Then,
\begin{align*}
    &\sum_{t=1}^T \|D_t\|_1 \leq \frac{d}{\sqrt\epsilon},\quad  \sum_{t=1}^T \|D_t\|^2 \leq \frac{d}{\epsilon}.
\end{align*}
\end{Lemma}

\begin{proof}
By the updating rule of \algo, $\hat v_{t-1}\leq \hat v_t$ for $\forall t$. Therefore, by the initialization $\hat v_0=0$, we have
\begin{align*}
    \sum_{t=1}^T \|D_t\|_1 &=\sum_{t=1}^T \sum_{i=1}^d (\frac{1}{\sqrt{\hat v_{t-1,i}+\epsilon}}-\frac{1}{\sqrt{\hat v_{t,i}+\epsilon}})\\
    &=\sum_{i=1}^d (\frac{1}{\sqrt{\hat v_{0,i}+\epsilon}}-\frac{1}{\sqrt{\hat v_{T,i}+\epsilon}})\\
    &\leq \frac{d}{\sqrt\epsilon}.
\end{align*}
For the sum of squared $l_2$ norm, note the fact that for $a\geq b>0$, it holds that
\begin{equation*}
    (a-b)^2\leq (a-b)(a+b)=a^2-b^2.
\end{equation*}
Thus,
\begin{align*}
    \sum_{t=1}^T \|D_t\|^2&=\sum_{t=1}^T \sum_{i=1}^d (\frac{1}{\sqrt{\hat v_{t-1,i}+\epsilon}}-\frac{1}{\sqrt{\hat v_{t,i}+\epsilon}})^2\\
    &\leq \sum_{t=1}^T \sum_{i=1}^d (\frac{1}{\hat v_{t-1,i}+\epsilon}-\frac{1}{\hat v_{t,i}+\epsilon})\\
    &\leq \frac{d}{\epsilon},
\end{align*}
which gives the desired result.
\end{proof}

\end{document}